\documentclass[10pt,journal,compsoc]{IEEEtran}

\ifCLASSOPTIONcompsoc
  \usepackage[nocompress]{cite}
\else
  \usepackage{cite}
\fi

\ifCLASSINFOpdf
\else
\fi

\usepackage{amsmath}
\usepackage{algorithmic}
\usepackage{array}

\ifCLASSOPTIONcompsoc
 \usepackage[caption=false,font=footnotesize,labelfont=sf,textfont=sf]{subfig}
\else
 \usepackage[caption=false,font=footnotesize]{subfig}
\fi

\usepackage[utf8]{inputenc} 
\usepackage[T1]{fontenc}    
\usepackage{hyperref}       
\usepackage{url}            
\usepackage{booktabs}       
\usepackage{amsfonts}       
\usepackage{nicefrac}       
\usepackage{microtype}      
\usepackage[table]{xcolor}
\usepackage[most]{tcolorbox}
\usepackage[linesnumbered,algoruled,boxed,lined]{algorithm2e}
\usepackage{enumitem}
\usepackage{multirow}
\usepackage{amssymb}
\usepackage{graphics}

\usepackage{tikz}
\usetikzlibrary{decorations.text}
\usetikzlibrary{arrows.meta,bending,decorations.markings}

\usepackage{amsmath,amsfonts,bm}

\def\eqref#1{(\ref{#1})}

\def\1{\bm{1}}

\def\rveta{{\boldsymbol{\eta}}}

\def\vone{{\bm{1}}}
\def\vmu{{\bm{\mu}}}
\def\vtheta{{\bm{\theta}}}

\def\vu{{\bm{u}}}
\def\vv{{\bm{v}}}

\def\vx{{\bm{x}}}

\def\mA{{\bm{A}}}
\def\mB{{\bm{B}}}
\def\mC{{\bm{C}}}
\def\mD{{\bm{D}}}
\def\mE{{\bm{E}}}
\def\mF{{\bm{F}}}

\def\mI{{\bm{I}}}

\def\mL{{\bm{L}}}

\def\mQ{{\bm{Q}}}
\def\mR{{\bm{R}}}

\def\mU{{\bm{U}}}
\def\mV{{\bm{V}}}
\def\mW{{\bm{W}}}
\def\mX{{\bm{X}}}
\def\mY{{\bm{Y}}}
\def\mZ{{\bm{Z}}}

\DeclareMathAlphabet{\mathsfit}{\encodingdefault}{\sfdefault}{m}{sl}
\SetMathAlphabet{\mathsfit}{bold}{\encodingdefault}{\sfdefault}{bx}{n}

\def\tL{{\tens{L}}}

\def\gD{{\mathcal{D}}}

\def\gG{{\mathcal{G}}}

\def\gL{{\mathcal{L}}}

\def\gT{{\mathcal{T}}}

\def\gW{{\mathcal{W}}}

\def\sB{{\mathbb{B}}}

\def\sL{{\mathbb{L}}}

\def\sV{{\mathbb{V}}}

\def\vnu{{\bm{\nu}}}
\def\tL{{\widetilde{\bm{L}_0}}}

\newcommand{\R}{\mathbb{R}}

\newtheorem{remark}{Remark}
\newtheorem{theorem}{Theorem}
\newtheorem{proof}{Proof}

\newcounter{bxincomm}
\definecolor{aqua}{rgb}{0.00,0.67,0.80}

\newcounter{gaocomm}
\definecolor{blue-violet}{rgb}{0.54, 0.17, 0.89}

\newcounter{ygcounter}

\newcommand{\ygc}[1]{\ygc{\stepcounter{ygcounter}{\bf [YG's comment \arabic{ygcounter}: #1]}\;}}

\definecolor{xuebin}{rgb}{0.7, 0.4, 1.0}

\newcounter{rkuncomm}
\definecolor{rkcolor}{rgb}{0.09,0.45,0.27}

\begin{document}
\title{Graph Denoising with Framelet Regularizer}
%
%
%
%

\author{Bingxin~Zhou,
        Ruikun~Li, 
        Xuebin~Zheng, 
        Yu~Guang~Wang,~\IEEEmembership{Member,~IEEE,}
        and~Junbin~Gao
\IEEEcompsocitemizethanks{\IEEEcompsocthanksitem B.Zhou is with The University of Sydney Business School, The University of Sydney, NSW, Australia and Institute of Natural Sciences, Shanghai Jiao Tong University, Shanghai, China. \protect\\ 
email: bzho3923@uni.sydney.edu.au
\IEEEcompsocthanksitem R.Li, X.Zheng and J.Gao are with The University of Sydney Business School, The University of Sydney, NSW, Australia.
\IEEEcompsocthanksitem Y.G.Wang is with Institute of Natural Sciences and School of Mathematical Sciences, Key Laboratory of Scientific and Engineering Computing of Ministry of Education, and Center for Mathematics of Artificial Intelligence Institute, Shanghai Jiao Tong University, Shanghai, China; Shanghai Artificial Intelligence Laboratory, Shanghai, China; and School of Mathematics and Statistics, The University of New South Wales, Australia.
}
\thanks{Manuscript received October 30, 2021.}}

\markboth{Journal of \LaTeX\ Class Files,~Vol.~14, No.~8, October~2021}%
{Zhou \MakeLowercase{\textit{et al.}}: Bare Demo of IEEEtran.cls for Computer Society Journals}
%



\IEEEtitleabstractindextext{%
\begin{abstract}
As graph data collected from the real world is merely noise-free, a practical representation of graphs should be robust to noise. Existing research usually focuses on feature smoothing but leaves the geometric structure untouched. Furthermore, most work takes $\mathbb{L}_2$-norm that pursues a global smoothness, which limits the expressivity of graph neural networks. This paper tailors regularizers for graph data in terms of both feature and structure noises, where the objective function is efficiently solved with the alternating direction method of multipliers (ADMM). The proposed scheme allows to take multiple layers without the concern of over-smoothing, and it guarantees convergence to the optimal solutions. Empirical study proves that our model achieves significantly better performance compared with popular graph convolutions even when the graph is heavily contaminated.
\end{abstract}

}

\maketitle

\IEEEdisplaynontitleabstractindextext

%
\IEEEpeerreviewmaketitle

\IEEEraisesectionheading{\section{Introduction}}
\label{sec:introduction}
\IEEEPARstart{D}{ata} quality is of crucial importance for a reliable modeling. When non-sampling errors become too large to neglect, the validity of the subsequent inference for the estimation result remains of limited confidence. Such concern has been explored throughout for grid data and panel data \cite{westerlund2007testing,to2009wavelet,goyal2020image}. As a special type of data with geometric structure and feature, graph data suffers from the same concern. For instance, when diagnose COVID-19, panicky suspect patients (nodes) might overstate their symptoms (feature), while some patients choose to undercover their track (edges). In the early stage of the outbreak, there could be a systematic error or unstandardized procedures that falsely classify 'low/medium/high risky regions', in which case the edge connection becomes less informative.

Graph Neural Networks (GNNs) have gained a predominant position in graph representation learning due to its promising empirical performance on various applications of machine learning. Recent studies in \cite{jin2020graph,zhu2021interpreting} showed that many graph convolution layers, such as \textsc{GCN} \cite{Kipf2017semi} and \textsc{GAT} \cite{velivckovic2017graph}, smooth the graph signal by intrinsically performing a $\sL_2$-based graph Laplacian denoising procedure. While the neighborhood-based local smoothness is desired for graph communities, such methods usually enforce global smoothness, which makes inter-cluster vertices indistinguishable, and the learned graph representation suffers from the well-known issue of \emph{over-smoothing}. Consequently, the model performance drops drastically before it learns the complicated relationship of graph geometry. In order to solve this problem, most researchers either consider spectral methods \cite{xu2018graph,zheng2021framelets} to include high-pass information, or introduce bias \cite{klicpera2018predict,xu2018representation} that balances between the fitness and smoothness of the models. 

While the graph feature smoothing has caught increasing attention, literature on structure denoising is far from sufficient. Researchers usually focus on the situation of defending graph adversarial attacks, however, such a scenario assumes a minor manipulation on the graph by hostile attackers, which does not universally exist in any case and does not threats all entities of a graph. On the other hand, a fundamental assumption on any GNNs, especially those popular message-passing networks, is that relevant vertices (or neighbors) share similar properties. When the precise connection is contaminated or indeterminant, the model estimation becomes unreliable. 

The lack of exploration on graph smoothing and denoising naturally brings our first research question: \textbf{\textit{How to measure the smoothness of both feature signal and structure space in a noisy graph?}} Instead of over-simplify or even neglect the graph geometry, the design should be tailor-made for a graph that reflects its key properties such as connectivity. The measure should also be aware of pitfalls that exist in training GNNs, such as over-smoothing. In consideration of the practicability, it's crucial to ask: \textbf{\textit{How to design a robust mechanism that is adaptive to any graph convolution layers?}} In this way, our method can be sliced to other graph convolutions to enhance their performance. 

To answer the questions, we explore regularizers on graphs and propose a framelet-based \textbf{\underline{Do}}uble-\textbf{\underline{T}}erm smoother, namely \textsc{DoT}, that detects and erases the noises in graph representation. We keep the key properties of graphs by applying the graph-norm on measuring the noise level. The punishment on each node is weighted by node degree, so that highly active nodes (with more neighbors) tolerant a smaller risk to be polluted. For feature noises, we replace the predominant $\sL_2$ norm with $\sL_1$ norm so that the sparsity in the transformed domain is desired while the smoothness is not overly pursued. The geometry noise, on the other hand, is measured by $\sL_{2,1}$ norm where we concentrate errors on merely connected nodes so that they are less likely to provoke un-neglectable estimation errors. We validate our model both theoretically and empirically. We detail the construction of our optimization target as well as its update rules. Its convergence analysis is attached to support its performance. We also compare three ablation models with popular convolution methods and several tasks, where our proposed method outperforms significantly. 

The rest of the paper is arranged as follows. Section~\ref{sec:relatedWork} discusses relevant research in literature.  Section~\ref{sec:l2_smooth} introduces popular graph convolution designs from the view of signal smoothing then Section~\ref{sec:denoising_preliminary} presents graph-related regularizers. Section~\ref{sec:ADMMdenoising} formulate the objective function for framelet-based \textsc{DoT} regularizer. The update rule and convergence analysis are provided under the protocol of the ADMM algorithm, following which three ablation models are listed. In Section~\ref{sec:exp}, several real-world examples indicate the model's recovery ability on different levels and types of contamination. We further conclude the work in Section~\ref{sec:conclusion}.

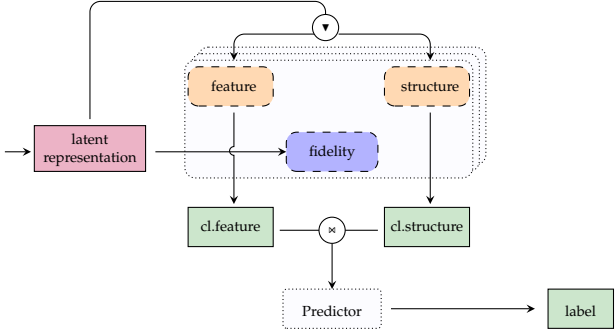
\begin{figure}
\centering
\newcommand\initialy{4}
\newcommand\nodeSize{0.75cm}

\tikzset{%
  tipSquare/.tip={Circle[open]}
}

\tikzset{%
  every neuron/.style={
    circle,
    draw,
    fill=white,
    scale=0.8,
    minimum size=\nodeSize
  },
  neuron missing/.style={
    draw=none, 
    scale=0.8,
    fill=$\dots$,
    text height=0cm,
    execute at begin node=$\dots$
  },
  snake it/.style={
    decorate, decoration=snake
  }
}

\begin{tikzpicture}[x=1.5cm, y=1.5cm, >=stealth, scale=0.58, every node/.style={transform shape}, curved arrow/.style={arc arrow={to pos #1 with length 2mm and options {}}},
reversed curved arrow/.style={arc arrow={to pos #1 with length 2mm and options reversed}}]  

\def\ystart{-2.4}
\def\xstart{-5.35}
\draw[rounded corners, densely dotted, fill=blue!1] (\xstart,\ystart) -- (\xstart+4.4,\ystart) -- (\xstart+4.4, \ystart-1.8) -- (\xstart, \ystart-1.8) -- cycle;

\draw[rounded corners, densely dotted, fill=blue!1] (\xstart-0.1,\ystart-0.1) -- (\xstart+4.3,\ystart-0.1) -- (\xstart+4.3, \ystart-1.9) -- (\xstart-0.1, \ystart-1.9) -- cycle;

\draw[rounded corners, densely dotted, fill=blue!1] (\xstart-0.2,\ystart-0.2) -- (\xstart+4.2,\ystart-0.2) -- (\xstart+4.2, \ystart-2.0) -- (\xstart-0.2, \ystart-2.0) -- cycle;

\def\ystart{-4.3}
\def\xstart{-7.85}
\draw[fill=purple!30] (\xstart,\ystart) -- (\xstart+1.7,\ystart) -- (\xstart+1.7, \ystart+0.7) -- (\xstart, \ystart+0.7) -- cycle;

\def\xstart{-4.0}
\draw[rounded corners, dashed, fill=blue!30] (\xstart,\ystart) -- (\xstart+1.4,\ystart) -- (\xstart+1.4, \ystart+0.6) -- (\xstart, \ystart+0.6) -- cycle;

\def\ystart{-3.3}
\def\xstart{-5.5}
\draw[rounded corners, dashed, fill=orange!30] (\xstart,\ystart) -- (\xstart+1.3,\ystart) -- (\xstart+1.3, \ystart+0.6) -- (\xstart, \ystart+0.6) -- cycle;
\draw[rounded corners, dashed, fill=orange!30] (\xstart+3,\ystart) -- (\xstart+4.3,\ystart) -- (\xstart+4.3, \ystart+0.6) -- (\xstart+3, \ystart+0.6) -- cycle;

\def\ystart{-5.45}
\draw[fill=black!50!green!20] (\xstart,\ystart) -- (\xstart+1.3,\ystart) -- (\xstart+1.3, \ystart+0.6) -- (\xstart, \ystart+0.6) -- cycle;
\draw[fill=black!50!green!20] (\xstart+3,\ystart) -- (\xstart+4.3,\ystart) -- (\xstart+4.3, \ystart+0.6) -- (\xstart+3, \ystart+0.6) -- cycle;

\def\ystart{-6.7}
\def\xstart{-4.05}
\draw[rounded corners=0.5mm, densely dotted, fill=blue!1] (\xstart,\ystart) -- (\xstart+1.5,\ystart) -- (\xstart+1.5, \ystart+0.6) -- (\xstart, \ystart+0.6) -- cycle;

\def\xstart{0.0}
\draw[fill=black!50!green!20] (\xstart,\ystart) -- (\xstart+1.0,\ystart) -- (\xstart+1.0, \ystart+0.6) -- (\xstart, \ystart+0.6) -- cycle;

\draw [->] (-8.3,-4.0) -- ++(0.4,0);
\draw [->] (-6.0,-4.0) -- ++(2.0,0);
\draw[rounded corners] (-7.0, -3.5) -- (-7.0, -1.7) -- (-3.4, -1.7) -- (-3.4, -1.9);
\draw [->] [rounded corners] (-3.5, -2.2) -- (-4.8, -2.2) -- (-4.8, -2.6);
\draw [->] [rounded corners] (-3.2, -2.2) -- (-1.8, -2.2) -- (-1.8, -2.6);
\draw[->] (-4.8,-3.4) -- (-4.8,-3.92) arc (90:270:0.1cm) -- (-4.8,-4.8);
\draw[->] (-1.8,-3.4) -- (-1.8,-4.8);

\draw (-4.1,-5.2) -- (-2.6,-5.2);
\draw[->] (-3.3,-5.4) --++ (0,-0.7);
\draw[->] (-2.4,-6.4) --++ (2.2,0);

\node[every neuron] at (-3.4, -2.1){$\blacktriangledown$};
\node[align=center, above] at (-7.0, -4.3) {latent\\representation};
\node[align=center, above] at (-4.8, -3.15) {feature};
\node[align=center, above] at (-1.8, -3.15) {structure};
\node[align=center, above] at (-3.3, -4.2) {fidelity};
\node[align=center, above] at (-4.85, -5.3) {cl.feature};
\node[every neuron] at (-3.3, -5.2){$\Join$};
\node[align=center, above] at (-1.85, -5.3) {cl.structure};
\node[align=center, above] at (-3.3, -6.6) {Predictor};
\node[align=center, above] at (0.5, -6.6) {label};

\end{tikzpicture}
\caption{Conceptual architecture of the proposed denoising scheme. For an input graph representation (either raw or latent), \textsc{DoT} slices its feature and structure information to do recursive cleaning. The fidelity function is extracted from the input noise signal to control the power of the cleaning scheme. The purified representation is then propagated to the next layer for label prediction (or other operations).}
\label{fig:architecture}
\end{figure}

\section{Related Work}
\label{sec:relatedWork}
\subsection{Graph Neural Networks for representation learning}
GNNs have shown great success in dealing with irregular graph-structured data that traditional deep learning methods such as CNNs fail to manage \cite{bronstein2017geometric,zhou2020graph,zhang2020deep,wu2020comprehensive}. The main factor that contributes to their success is that GNNs are able to learn the structure pattern of the graph while CNNs can only handle regular grid-like structures. 

A broad spectrum of GNN approaches has been proposed for defining the graph convolutional operation, where the two main groups are spatial-based and spectral-based methods. Spatial-based methods is dominantly researched and practiced due to its intuitive characteristics \cite{Kipf2017semi,velivckovic2017graph,xu2018powerful,xu2018representation,jiang2021co}. As an extension of convolutional neural networks (CNNs), spatial-based linear GNNs pass messages along the graph Laplacian, which naturally performs local smoothness on second-order graph differences \cite{zhu2021interpreting,chen2021graph} and can be interpreted as low-pass filters of graph signals \cite{nt2019revisiting}. Spectral methods that convert the raw signal or features in the vertex domain into the frequency domain, on the other hand, was paid less attention to \cite{bruna2013spectral,henaff2015deep,defferrard2016convolutional,levie2018cayleynets,xu2018graph,zheng2021framelets}. In fact, spectral-based methods have already been proved to have a solid mathematical foundation in graph signal processing \cite{shuman2013emerging}. Versatile Fourier \cite{defferrard2016convolutional,Kipf2017semi,henaff2015deep}, wavelet transforms\cite{xu2018graph,zheng2020mathnet} and framelets\cite{zheng2021framelets} have also shown their capabilities in graph representation learning. In addition, with fast transforms being available in computing strategy, a big concern related to efficiency could be well resolved.

\subsection{Graph Signal Smoothing and Denoising}
Common to both categories of graph convolutions is the signal smoothing of graph filters. Spatial methods usually implement polynomial \cite{Kipf2017semi,tremblay2018design}, Lanczos \cite{liao2018lanczosnet} or ARMA \cite{isufi2016autoregressive,thomas2021higher} filters. From the perspective of graph signal denoising, an aggregation layer approximates smooth graph signals with pairwise quadratic local variation minimization \cite{zhou2004learning,fu2020understanding}. While linear approximation achieves fast optimization over graph denoising, the abandoned details at each layer limit the model's fitting power over complex nonlinear functions, which contradicts the requirement of an effective deep neural network. In contrast, stacking multiple network layers aggregate complicated relationships of a large receptive field, but taking a weighted moving average over too many neighbors could lose sharp changes of frequency response. 

Instead, spectral methods, due to their close connection to signal processing, possess detail-preserved global smoothness. Conventional spectral transforms has been widely used for processing non-stationary signals, such as images compression and restoration \cite{vaseghi2008advanced,parrilli2011nonlocal}, speech enhancement \cite{bahoura2006wavelet,loizou2007speech} and so on. Compared to the Fourier transform of the time-frequency domain, wavelets transforms engage between the time and scale domain of signals, which facilitates a multi-scale or multi-resolution view of the input. The signal coefficients exhibit sparse distribution while noise coefficients spread uniformly with a small amplitude. This property allows differentiate signal and noise with threshold-based approaches \cite{chang2000adaptive,aminghafari2006multivariate} or empirical risk minimization approaches \cite{ramani2008monte,ding2015artifact}.

The same philosophy of sparsifying signal representation is implanted to smoothing spectral graph convolutions. Compared to spatial-based methods where the sparsity is usually restricted in vertex domain \cite{liu2021elastic}, spectral methods measure the graph sparsity in transformed domain, e.g., in Fourier \cite{chen2021bigcn} or Framelet domain \cite{dong2017sparse}. Usually $\sL_1$ regularization is suggested as the sparsity measurement. 

ADMM \cite{gabay1976dual} or Split-Bregman \cite{goldstein2009split} are conventional solutions of recursive optimization. While the authors of \cite{xie2019differentiable,li2020training} show that the the feed-forward propagation of neural networks can be interpreted as optimization iterations, solving the optimal graph representation over diverse prior design becomes possible, such as sparsity over input signals \cite{liu2021elastic} or first-order difference of input \cite{wang2015trend}.

\section[L2 smoothing]{Graphs Convolutions and $\sL_2$ Smoothers}
\label{sec:l2_smooth}
This section discusses the equivalent expression of graph convolution from the perspective of graph signal denoising. We consider an undirected graph $\gG=(\sV,\mathbb{E},\mX)$ with $N=|\sV|$ nodes. The edge connection is described by an adjacency matrix $\mA\in \R^{N\times N}$ and the $d$-dimensional node feature is stored in $\mX\in\R^{N\times d}$. Here we neglect the nonlinear activation function $\sigma(\cdot)$ for simplicity.

\subsection{Spatial Graph Convolution}
A typical way of denoising graph signal is via graph smoothing, where the feature smoothness can be forced by a modified \emph{Dirichlet energy function}. While the smoothed representation should be close to the input representation, one leverages a \emph{fidelity term} to guarantee this similarity, and this similarity is usually measured by a convex distance measure, such as the Euclidean distance ($\sL_2$ norm).

\begin{theorem}
    Consider a graph smoothing problem on the input representation $\mX\mW$
    \begin{equation} \label{func:spatial_smoothing}
        \min_{\mU} {\rm tr}(\mU^{\top}\tilde{\mL}\mU)+\|\mU-\mX\mW\|^2_{2},
    \end{equation}
    where the smoothness of graph is measured by the Dirichlet energy that ${\rm tr}(\mU^{\top}\tilde{\mL}\mU)=\frac12\sum_{i,j=1}^N \mA_{ij}(\frac{\mX_i}{\sqrt{\mD_{ii}}}-\frac{\mX_j}{\sqrt{\mD_{jj}}})^2$. Here the $\tilde{\mL}:=\mI-\mD^{-\frac12}\mL\mD^{-\frac12}=\mI-\tilde{\mA}$ denotes the normalized graph Laplacian with respect to the adjacency matrix $\mA$ or the degree matrix of the graph $\mD$. The first-order approximation to this problem's optimal solution is $\mU:=\tilde{\mA}\mX\mW$.
\end{theorem}

\begin{proof}
    To obtain the optimal solution to the objective function, we calculate the partial derivatives
    \begin{align*}
        && \frac{\partial {\rm tr}(\mU^{\top}\tilde{\mL}\mU)+\|\mU-\mX\mW\|^2_{2}}{\partial \mU} = 0,& \\
        & \Rightarrow\;\; & 2\tilde{\mL}\mU+2(\mU-\mX\mW) = 0,& \\
        & \Rightarrow\;\; & (\mI+\tilde{\mL})\mU = \mX\mW,&  \\
        & \Rightarrow\;\; & \mU = (\mI+\tilde{\mL})^{-1} \mX\mW.& 
    \end{align*}
    The first-order Taylor expansion gives that $(\mI+\tilde{\mL})^{-1} \approx \mI-\tilde{\mL}$. By Definition, $\mL=\mD-\mA$ and $\tilde{\mL}=\mI-\tilde{\mA}$, which suggests the first-order approximation of the optimal result $\mU \approx \tilde{\mA}\mX\mW$.
\end{proof}

This approximation in graph representation learning is known as graph convolutional networks \cite{Kipf2017semi}, which is a popular message-passing scheme that averages one-hop neighborhood representation with normalized adjacency matrix $\tilde{\mA}$. The $\mW$ is a learnable weight matrix that embeds the high-dimensional raw input $\mX$ to a lower dimension. There are many other aggregation strategies by different spatial-based GNNs under the message-passing \cite{Gilmer_etal2017} framework, such as \textsc{GAT} \cite{velivckovic2017graph}, \textsc{APPNP} \cite{klicpera2018predict} and \textsc{SGC} \cite{wu2019simplifying}. Though, most of them implicitly achieves this optimization objective. Authors of \cite{nt2019revisiting} explored the equivalence of GCN's propagation process to low-pass filters that smooths the graph signal by brutally filtering out all the detailed information. Other recent research \cite{zhu2021interpreting,liu2021elastic} identified that many other message passing schemes, although designing different propagation rules, share the same construction logic of the objective function in terms of denoising graph signals under the smoothness and fitness constraints \cite{zhou2004learning}, which is tricky to balance.

\subsection{Spectral Graph Convolution}
Compared to spatial counterparts, \emph{Spectral graph convolution}s transform graph signals into frequency domain for further operation, i.e., graph smoothing. 

\begin{theorem}
    For a given transform $\boldsymbol{\Phi}$, a spectral-based graph convolution smooths the input signal $\mX$ by optimizing
    \begin{equation} \label{func:spectral_smoothing}
        \min_{\mU} {\rm tr}(\mC^{\top} {\rm diag}(\vtheta) \mC)+\|\mU-\mX\|^2_2,
    \end{equation}
    where $\mC:=\boldsymbol{\Phi}\mU$ denotes the transformed coefficients with respect to the clean signal $\mU$ in frequency domain. The optimal solution to \eqref{func:spectral_smoothing} is $\vtheta\star \mX = \boldsymbol{\Phi}^{-1}{\rm diag}(\vtheta) \boldsymbol{\Phi} \mX$.
\end{theorem}

\begin{proof}
    We first rewrite the equation by replacing $\mC$ with $\boldsymbol{\Phi}\mU$, which gives
    \begin{align*}
        \min_{\mU} {\rm tr}(\mU^{\top}\boldsymbol{\Phi}^{\top} {\rm diag}(\vtheta) \boldsymbol{\Phi}\mU)+\|\mU-\mX\|^2_2
    \end{align*}
    Note that the graph Laplacian $\mL=\mV^{\top}\Lambda\mV=\boldsymbol{\Phi}^{\top}{\rm diag}(\vtheta)\boldsymbol{\Phi}$ where $\{(\vv_i,\lambda_i):\vv_i\in\mV,\lambda_i\in\Lambda\}$ are eigenpairs with respect to $\mL$. We then have 
    $$
    {\rm tr}(\mU^{\top}\boldsymbol{\Phi}^{\top} {\rm diag}(\vtheta) \boldsymbol{\Phi}\mU)={\rm tr}(\mU^{\top}\mL\mU)+\|\mU-\mX\|^2_2.
    $$
    Similar to above, the partial derivatives of this function brings the optimal solution that $\mU=\mL\mX=\boldsymbol{\Phi}^{\top}{\rm diag}(\vtheta)\boldsymbol{\Phi}\mX$, which is identical to the convolution function \eqref{func:spectral_conv} below. 
\end{proof}

The optimal solution to the above smoothing problem is identical to spectral graph convolution. Formally, a layer of spectral graph convolution reads
\begin{equation} \label{func:spectral_conv}
    \vtheta\star \mX = \boldsymbol{\Phi}^{-1}{\rm diag}(\vtheta) \boldsymbol{\Phi} \mX,
\end{equation}
where the basis $\boldsymbol{\Phi}$ is closely related to the eigenvectors of graph Laplacian and is determined by the type of transform, such as Fourier \cite{bruna2013spectral}, Wavelet \cite{xu2019graph,zheng2020mathnet} and Framelet \cite{zheng2021framelets} transforms. Compared to spatial convolutions, this genre allows graph signal processing \cite{ortega2018graph} on multiple channels, which pursues both local and global smoothness. 

As the coefficient are graph signals projected to different channels by low-pass or high-pass filters, it naturally separates the local and global smoothness procedure, which makes the over-smoothing and loss-of-expressivity issue less of a concern. That has been said, the recovered graph signal still to some extend suffers from the global over-smoothing issue from $\sL_2$ regularizer, as it suppresses regional fluctuations. To solve this issue, we refer to a similar denoising idea adopted in signal processing. Meanwhile, as neither of the above functions pay explicit attention to structure information, we formulate connection-related noise measurements to fill this gap.

\section{The Noise of Graph Signals}
\label{sec:denoising_preliminary}
This section formulates the regularization schemes related to graph smoothing and subspace extraction to answer the first research question. We start from the conventional analysis model for graph signal denoising and then introduce regularizers that filter the feature and structure noises. 

\subsection{Analysis-Based Graph Regularization}
We assume the observed graph signal $\vx=\vu+\epsilon$ constitutes a graph function $\vu:\sV\mapsto\R$ and Gaussian white noise $\epsilon$. A denoising task recovers the true graph signal $\vu$ from the observation $\vx$. We restrict a sparse representation of graphs in the transformed domain and follow the \emph{analysis based model} \cite{cai2010split} to measure the sparsity priors (that is, the coefficient matrix in the transformed domain is assumed sparse). The unconstrained minimization is set to
\begin{align*}
    \min_{\vu} \|\gD\vu\|_1+H(\vu),
\end{align*}
where $\gD\in\R^{m\times N}$ is a linear transform generated from discrete transforms, such as discrete Fourier transforms or Framelet transforms. Its sparsity is constrained by the $\sL_1$-norm. The $H(\cdot)$ is a smooth convex function that measures the data fidelity, a common choice of which is the empirical risk minimization $H(\vu)=\|\mB\vu-\vx\|_2$ with some linear transformation $\mB$ and observations $\vx$.

\subsection{Penalty Functions Design}
As both node and edge noise are assumed in observed data, we design explicit regularizers individually to reveal the noiseless set of graph signal $\mU$ and adjacency matrix $\mZ$.

\subsubsection{Vertex Feature with Spectral Sparsity}
We start from node feature denoising with $\sL_1$ wavelet regularization. The sparsity and similarity of the signal approximation is measured by
\begin{equation}\label{func:objective_node_init}
    \min_{\mU} \|\gD\mU\|_{1,G}+\frac{1}{2}\|\mU-\mX\|^2_{2,G},
\end{equation}
For a graph signal $\vu$, instead of the Euclidean $\sL_p$-norm, we define a graph $\sL_p$-norm $\|\vu\|_{p,G}:=\left(\sum_i|\vu[i]|^p \cdot d[i]\right)^{\frac1p}$, where $d[i]$ indicates the degree of the $i$th node with respect to $\vu$. This work considers $p=1,2$. One can interpret the graph norm as a weighted-sum version of regularizer, where errors associated to high-degree nodes result in heavier penalties.

\subsubsection{Edge Connection with Self-Expressiveness}
One fundamental assumption in graph convolution design is that connected nodes are more likely to share common characteristics, and they can be useful to distinguish one group of nodes from another. A similar idea is adopted in sparse subspace clustering, where inner-group information flow is described by self-expressiveness. We thus leverage this idea and assume that every node of a graph can be written as a sparse linear combination of its neighbor nodes. For contaminant data with a noise $\mE \in \R^{N\times d}$, we restrict
\begin{align} \label{func:objective:edge_init}
    \min_{\mZ,\mE}&\|\mZ\|_{1}+\|\mE\|_{2,1,G}  \notag\\
    s.t.&~~ \mU=\mZ\mU+\mE, {\rm diag} (\mZ)=0.
\end{align}
Here the sum error $\|\mE\|_{2,1,G} =\sum_{i=1}^{N}D_i\sqrt{\sum_{j=1}^{D_i}|E_{i,j}|^2}$ is weighted by $D_i$, the degree of node $i$. This convex optimization finds an effective sparse representation of node connectivity $\mZ\in\R^{N\times N}$, which can be served as a noiseless graph adjacency matrix. We regularize the error term $\mE$ by a $\sL_{2,1,G}$-norm to concentrates errors on features of small-degree nodes.

\subsection{Discrete Framelet Transform}
As discussed earlier, a graph signal measures its noise level by the sparsity of the transformed graph coefficients. This work takes fast undecimated graph Framelet transforms \cite{dong2017sparse,zheng2021framelets} for multi-channel fast denoising. Formally, the Wavelet Frame, or Framelet, is defined by a filter bank $\rveta:=\{a;b^{(1)},\dots,b^{(K)}\}$ (with $K$ the number of high-pass filters) and the eigen-pairs $\{(\lambda_j,\vu_j)\}_{j=1}^{N}$ of its graph Laplacian $\mL$. The $a$ and $b^{(k)}$ are called the \emph{low-pass} and \emph{high-pass filters} of the Framelet transforms, which preserves the approximation and detail information of the graph signal, respectively. At \emph{scale level} $l=1,\ldots,L$, the low-pass and the $k$th ($k=1,\dots,K$) high pass undecimated Framelet basis at node $p$ are defined as 
\begin{align*} 
    \boldsymbol{\varphi}_{l,p}(v) &= \sum_{\ell=1}^{N} \hat{\alpha}\left(\frac{\lambda_{\ell}}{2^{l}}\right)
    \overline{\vu_{\ell}(p)}\vu_{\ell}(v); \\
    \boldsymbol{\psi}_{l,p}^k(v) &= \sum_{\ell=1}^{N} \widehat{b^{(k)}}\left(\frac{\lambda_{\ell}}{2^{l}}\right)\overline{\vu_{\ell}(p)}\vu_{\ell}(v),
\end{align*}
where $\{\alpha;\beta^{(1)},\dots,\beta^{(K)}\}$ are the associated scaling functions defined by $\rveta$. The corresponding Framelet coefficients for node $p$ at scale $l$ of a given signal $\vx$ are the projections $\langle\boldsymbol{\varphi}_{l,p},\vx\rangle$ and $\langle\boldsymbol{\psi}_{l,p}^k,\vx\rangle$. Note that we formulate the transforms with Haar-type filters \cite{dong2017sparse} with dilation factor $2^l$ to allow efficient transforms.

The computational cost of decomposition (and reconstruction) algorithm of Framelet transforms can be heavy. We thus consider approximating the scaling functions $\alpha$ and $\{\beta^{(1)},\dots,\beta^{(K)}\}$ by Chebyshev polynomials up to $m$-order, respectively, denoted by $\gT^m_0$ and $\{\gT^m_{k}\}_{k=1}^{K}$. For a given level $L$, we define the full set of Framelet coefficients
\begin{align*}
    \gW_{k,1}\vx &= \gT^m_k(2^{-H}\mL)\vx
\end{align*}
at $l=1$. When $l=2,\dots,L$, the coefficients 
\begin{align*}   
    \gW_{k,l}\vx &= {\gT^m_k}(2^{-H -l}\mL){\gT^m_0}(2^{-H-l+1}\mL){\cdots\gT^m_0}(2^{-H}\mL)\vx,
\end{align*}
where the dilation scale $H$ satisfies $\lambda_{\max} \leq 2^H\pi$. In this definition, the finest scale is $1/2^{H+L}$ that guarantees $\lambda_{\ell}/2^{H+L-l}\in(0,\pi)$ for $\ell=1, 2, ..., N.$

\section{A Layer of Graph Feature and Structure Denoising}
\label{sec:ADMMdenoising}
This section constructs the graph denoising scheme concerning the second research question. The signal sparsity is measured under the undecimated Framelet transform system that transforms the graph signal to low-pass and high-passes Framelet coefficients with the $m$-degree Chebyshev polynomial of $\mL$ by a set of $N\times N$ orthonormal decomposition operator $\boldsymbol{\gW}$. The sequence length $(KL+1)$ is determined by the number of high-pass filters $K$ and scale level $L$.

\subsection{ADMM Undecimated Framelet Denoising Model}
We now present the denoising layer design. Here we interpret the forward propagation of a GNN layer as an optimization iteration. The input $\mX$ is a noisy representation of the raw signal or its hidden representation. Our target is to recover $\mU$ and $\mZ$, a clean set of graph feature and structure representations that is free from pollution. The objective function and the corresponding update rule is detailed below.

\subsubsection{Objective Function Design}
The objective function that penalizes both node and edge noises reads
\begin{align} \label{func:objective_full_init}
    \min_{\mU,\mZ,\mE}& \|\vnu\boldsymbol{\gW}\mU\|_{1,G}+\|\mZ\|_{1}+\lambda_1\|\mE\|_{2,1,G}+\frac{\lambda_2}2\|\mU-\mX\|^2_{2,G}, \notag\\
    s.t.&~~\operatorname{diag}(\mZ)=0,\; \mU = \mZ\mU+\mE, \;\mZ \vone = \vone. 
\end{align}
Here $\|\vnu\boldsymbol{\gW}\mU\|_{1,G}:=\sum_{(k,l)\in\sB}\nu_{k,l}\|\gW_{k,l}\mU\|_{1,G}$ promotes the sparsity of the Framelet coefficients of the $k$th high-pass element on the $l$th scale level, where $\sB:=\{(0,L)\}\bigcup\{(k,l),1\leq k\leq K, 1\leq l\leq L\}$. The $\vnu$ is a set of tuning parameters with respect to tight decomposition operator $\gW_{k,l}$ in high pass. 
The $\mU\in\R^{N\times d}$ is the target (noiseless) signal approximation of $N$ nodes, $\mL$ is the associated graph Laplacian, and $\mX$ is the noisy input representation. 

The above function can be rewritten with a variable splitting strategy. Define $\mY = \mZ - \text{diag}(\mZ)$, $\mQ := \boldsymbol{\gW}\mU$, i.e., $\mQ_{k,l} := \gW_{k,l}\mU \in \R^{N\times d}$, gives
\begin{align} \label{func:objective_full}
    \min_{\mU,\mZ,\mE}& \|\vnu\mQ\|_{1,G}+\|\mZ\|_{1}+\lambda_1\|\mE\|_{2,1,G}+\frac{\lambda_2}2\|\mU-\mX\|^2_{2,G}, \notag\\
    s.t.&~~ \mU = \mY\mU+\mE, \;\mY \vone = \vone, \; \mY = \mZ - \text{diag}(\mZ). 
\end{align}
The associated augmented Lagrangian \cite{nocedal2006numerical} reads 
\begin{align} \label{func:objective_full_lagrangian}
    &\gL(\mU,\mZ,\mE,\mQ, \mY;\Lambda_1,\Lambda_2,\Lambda_3,\Lambda_4,\mu_1,\mu_2,\mu_3, \mu_4) \notag \\
    =& \sum_{k,l}\nu_{k,l}\|\mQ_{k,l}\|_{1,G}+\|\mZ\|_{1}+\lambda_1\|\mE\|_{2,1,G}+\frac{\lambda_2}2\|\mU-\mX\|^2_{2,G} \notag\\
    &+\frac{\mu_1}{2}\|\mU-\mY\mU -\mE\|_2^2    +\sum_{k,l}\frac{\mu_{2}}{2}\|\mQ_{k,l}-\gW_{k,l}\mU\|_2^2 \notag \\
    &+\frac{\mu_3}{2}\|\mY \vone-\vone\|_2^2 + \frac{\mu_4}2\|\mY-\mZ+\text{diag}(\mZ)\|^2_2 \notag \\ 
    & +\text{tr}\left(\Lambda_1^{\top}(\mU-\mY\mU-\mE)\right)+\sum_{k,l}\text{tr}\left(\Lambda_{2;k,l}^{\top}(\mQ_{k,l}-\gW_{k,l}\mU)\right) \notag \\ 
    &+\Lambda_{3}^{\top}(\mY \vone-\vone) +\text{tr}(\Lambda^{\top}_4(\mY - \mZ + \text{diag}(\mZ))). 
\end{align}
We call $\Lambda_1, \Lambda_{2;k,l}\in\R^{N\times d}; \Lambda_3\in\R^{N};  \Lambda_4\in\R^{N\times N}$ the Lagrangian multipliers. The $\Lambda_i$ ($i=1,2,3,4$) can be considered as the running sum of errors associated with its constraint, and the target is to choose the optimal $\Lambda_i$ that minimize the residual of the $i$th constraint. 
The augmented Lagrangian penalty parameters $\mu_1,\mu_{2;k,l},\mu_3, \mu_4$ are positive by definition. They stand for upper bounds over all the adaptive penalty parameters. In practice we value an identical $\mu_2$ to $\mu_{2;k,l}$ for faster matrix inversion in updating $\mU$. We will omit the subscripts for the rest of paper. 

\begin{remark}
    Conventional splitting methods select $\mu_i$s in advance. This paper leverages adaptive tuning methods that only require a proper initialization on $\mu^{(0)}_i$s. 
\end{remark}

\begin{remark}
    For simplicity, we can define the low-pass $\vnu_{0,L}=0$ and high-passse $\vnu_{k,l}=4^{-l-1}\nu_0$ by initializing $\nu_0$ to avoid the overwhelming fine-tuning work in $\vnu$.
\end{remark}
\begin{remark}
    We keep the $\text{diag}(\mZ^{(t)})$ term here for theoretical completeness. For implementation we remove the term $\text{diag}(\mZ^{(t)})$ as it equals zero. 
\end{remark}

\begin{remark}
    With $\tL:=\mI-\mY$, we can rewrite $\mU-\mY\mU-\mE = \tL\mU-\mE$, and $\tL$ denotes the normalized Laplacian of noiseless adjacency matrix.
\end{remark}

\subsubsection{ADMM Update Scheme}\label{Sec:ADMM}
The objective function is optimized with the alternating direction method of multipliers (ADMM) \cite{gabay1976dual}. There involves six (sets of) parameters to update. This work updates the primal variables $\mU,\mZ,\mE,\mY,\mQ$ iteratively, but it is possible to design a parallel update scheme, see \cite{deng2017parallel}. 
At the $t+1$th iteration, we update iteratively 
\begin{enumerate}[leftmargin=*]
    \item $\mU^{(t+1)}$
    \begin{flalign} \label{func:objective_1_u}
        =&\left(\lambda_2\mD+\mu^{(t)}_1\tL^{(t)\top}\tL^{(t)}+\mu_{2}^{(t)}\mI\right)^{-1}& \notag\\
        &\left(\mu^{(t)}_1\tL^{(t)\top}\mE^{(t)}+\lambda_2\mD\mX-\tL^{(t)\top}\Lambda_1^{(t)} \right.& \notag\\
        &\left.+\sum_{k,l}(\mu_{2}^{(t)}\gW_{k,l}^{\top}\mQ^{(t)}_{k,l}+\gW_{k,l}^{\top}\Lambda_{2;k,l})\right)&
    \end{flalign}
    where the noiseless graph Laplacian $\tL^{(t)}$ is updated with $\mY^{(t)}$. The matrix inversion reaches a precise solution by the linear solver \cite{Golub1996matrix} or the Cholesky Factorization \cite{Haddad2009Cholesky}. 
    \item $\mZ^{(t+1)} = \mR - \text{diag}(\mR)$
    \begin{align} \label{func:objective_1_z}
        \text{ where }&\mR := \mathcal{T}_{1/\mu^{(t)}_4}\left(\mY^{(t)} + {\mu^{(t)}_4}^{-1}\Lambda^{(t)}_4\right).
    \end{align}
    The $\mathcal{T}_{\eta}(x) = \text{sign}(x)\max\{|x|-\eta,0\}=\text{ReLU}(x-\eta)-\text{ReLU}(-x-\eta)$ 
    is a \emph{soft-threshold operator}.
    \item $\mE_i^{(t+1)}$
    \begin{flalign} \label{func:objective_1_e}
        &=\mathcal{T}_{1/\mu_1^{(t)}}^i\left(\tL^{(t+1)}\mU^{(t+1)}+{\Lambda_1^{(t)}}/{\mu_1^{(t)}}\right).&
    \end{flalign}
    The above update rule works on the $i$th row of $\mE$. The $\mathcal{T}_{\eta}^i(x) = ({x_i}/{\|x_i\|_2})\max\{\|x_i\|_2-\eta,0\}$ denotes a row-wise soft thresholding for group $\sL_2$-regularization. 
    \item $\mY^{(t+1)}$ 
    \begin{align} \label{func:objective_1_y}
        \hspace{-4mm}&=\left(\mu^{(t)}_1(\mU^{(t+1)}-\mE^{(t+1)})\mU^{(t+1)\top} + \mu^{(t)}_3\vone\vone^{\top} + \mu^{(t)}_4\mZ^{(t+1)} \right. \notag \\
        \hspace{-4mm}&\left. +\Lambda^{(t)}_1\mU^{(t+1)\top} - \Lambda^{(t)}_3\vone^{\top}-\Lambda^{(t)}_4\right) \left({\mu^{(t)}_4}^{-1}\mI  \right. \\
        \hspace{-4mm}&\left.-{\mu^{(t+1)}_4}^{-1} \widetilde{\mU}^{(t+1)}\left[\mu^{(t+1)}_4 \mI + \widetilde{\mU}^{(t+1)\top} \widetilde{\mU}^{(t+1)}\right]^{-1} \widetilde{\mU}^{(t+1)\top}\right). \notag
    \end{align}
    The $\widetilde{\mU}^{(t+1)} =[\sqrt{\mu^{(t+1)}_1}\mU^{(t+1)}, \sqrt{\mu^{(t)}_3}\vone] \in\mathbb{R}^{N\times (d+1)}$. The $\vone$ is a $N$ dimensional one vector. We assume $N\gg d$, i.e., the graph size is much larger than the node feature dimension.
    \item $\mQ_{k,l}^{(t+1)}[i,:]$
    \begin{align} \label{func:objective_1_q}
        \hspace{-4mm}=\mathcal{T}_{\nu_{k,l}d_i/\mu_{2}^{(t)}}\left(\gW_{k,l}\mU^{(t+1)}[i,:]-{\mu^{(t)}_{2}}^{-1}\Lambda_{2;k,l}^{(t)}[i,:]\right).
    \end{align}
    The above update rule acts row-wisely where $\mQ_{k,l}^{(t+1)}[i,:]$ denotes the $i$th row of $\mQ_{k,l}^{(t+1)}$.
    \item \textit{Update $\Lambda^{(t+1)}$s and $\mu^{(t+1)}$s:}
    \begin{align} 
        \Lambda_{1}^{(t+1)}&=\Lambda_{1}^{(t)}+\mu_{1}^{(t)}(\mU^{(t+1)} - \mY^{(t+1)}\mU^{(t+1)}-\mE^{(t+1)}); \notag \\
        \Lambda_{2;k,l}^{(t+1)}&=\Lambda_{2;k,l}^{(t)}+\mu_{2}^{(t)}(\mQ_{k,l}^{(t+1)}- \notag \gW_{k,l}\mU^{(t+1)}); \\
        \Lambda_{3}^{(t+1)}&=\Lambda_{3}^{(t)}+\mu_{3}^{(t)}(\mY^{(t+1)} \vone-\vone); \notag  \\
        \Lambda_{4}^{(t+1)}&=\Lambda_{4}^{(t)}+\mu_{4}^{(t)}\left(\mY^{(t+1)} - \mZ^{(t+1)} + \text{diag}(\mZ^{(t+1)})\right); \notag \\
        \mu_{i}^{(t+1)}&=\min\left(\rho\mu_{i}^{(t)}, \mu_{i,\max}\right), \quad i = 1,2,3,4.
        \label{func:objective_1_mu_lambda}
    \end{align}
\end{enumerate}

The pseudo-code is summarized in Algorithm~\ref{algo:ADMMDenoise}.
\begin{algorithm}[t]
    \hsize=\textwidth
    \SetKwData{step}{Step}
    \SetKwInOut{Input}{Input}\SetKwInOut{Output}{Output}
    \BlankLine
    \Input{iteration $t$, scale level $L$, high-pass number $K$}
    \Output{$\mU^{(t)}, \mZ^{(t)}$}
    Initialization: $\sB=\{(0,L))\}\bigcup\{(k,l),1\leq k \leq K,1\leq l\leq L\}$\\
    Initialization: $\mY^{(0)}, \mZ^{(0)}, \mU^{(0)}, \mQ^{(0)}, \mE^{(0)}; \Lambda^{(0)}s, \mu^{(0)}s$\\
    \For{$k \leftarrow 1$ \KwTo $K$}{
        \For{$(k,l) \in \sB$}{
        Update $\mU^{(t)},\mZ^{(t)},\mE^{(t)},\mY^{(t)},\mQ_{k,l}^{(t)}$ by \eqref{func:objective_1_u}-\eqref{func:objective_1_e};\\
        Update $\Lambda_1^{(t)},\Lambda_{2;k,l}^{(t)},\Lambda_3^{(t)},\Lambda_4^{(t)}$ by \eqref{func:objective_1_mu_lambda};\\
        Update $\mu_1^{(t)},\mu_{2}^{(t)},\mu_3^{(t)},\mu_4^{(t)}$ by \eqref{func:objective_1_mu_lambda}
        }
    }
    Update loss function $\gL$ by \eqref{func:objective_full_lagrangian}.
    \caption{\textsc{DoT} for graph denoising}
    \label{algo:ADMMDenoise}
\end{algorithm}

\subsection{Convergence of the main algorithm}
The original optimization problem \eqref{func:objective_full_init}, in our proposed scheme, is considered as a component of neural network. In the last section, we leverage the ADMM scheme to solve the augmented Lagrangian problem \eqref{func:objective_full_lagrangian} where the optimal $\mZ$ and $\mU$ are both outputs of network layers.  The $t$th iteration of the ADMM update scheme, in particular, acts as the $t$th layer of neural network that updates from $\{\mU^{(t-1)},\mZ^{(t-1)}\}$ to $\{\mU^{(t)},\mZ^{(t)}\}$. This section provides the theorem that guarantees the convergence of ADMM algorithm for \eqref{func:objective_full_lagrangian}. 

\begin{theorem}\label{theorem:convergence}
Let $\{\Gamma_t= (\mU^{(t)}, \mZ^{(t)}, \mE^{(t)}, \mQ^{(t)}, \mY^{(t)}, \Lambda^{(t)}_1,\Lambda^{(t)}_2,$ $\Lambda^{(t)}_3,\Lambda^{(t)}_4)\}^{\infty}_{t=1}$ be the sequence generated  by the ADMM scheme \eqref{func:objective_1_u}-\eqref{func:objective_1_mu_lambda} presented in subsection \ref{Sec:ADMM}. Under the assumption that $\mU^{(t)}$ is bounded, the sequence $\{\Gamma_t\}$ satisfies the following properties:
\begin{enumerate}
\item The generated sequence $\{\Gamma_t\}_{t=1}^{\infty}$ is bounded.
\item  The sequence $\{\Gamma_t\}_{t=1}^{\infty}$ has at least one accumulation point $\Gamma_* =(\mU_*, \mZ_*,  \mE_*, \mQ_*, \mY_*, \Lambda_{1*},\Lambda_{2*}, \Lambda_{3*},\Lambda_{4*})$ that satisfies first-order optimality KKT conditions:  
\begin{align*}
    &\mU_* = \mY_* \mU_* + \mE_*; \\ 
    &\mathcal{Q}_* = \mathcal{W}_*\mU_*;\\
    &\mY_*\mathbf{1} = \mathbf 1; \\
    &\mY_* = \mZ_* - \text{diag}(\mZ_*); \\
    &\Lambda_{1*}\in\lambda_1\partial_{\mE}\|\mE_*\|_{2,1,G};\\
    &-\Lambda_{2*; k,l}\in \nu_{k,l}\partial_{Q_{k,l}}\|\mQ_{*;k,l}\|_{1,G};\\ 
    &\Lambda_{1*}\mU^T_{*}+\Lambda_{3*}\mathbf 1^T - \Lambda_{4*} = 0;\\
    &\Lambda_{4*} \in \partial_{\mZ}\|\mZ_*\|_1.
\end{align*}
\item  $\{\mU^{(t)}\}, \{\mZ^{(t)}\},  \{\mE^{(t)}\}, \{\mQ^{(t)}\}$ and $\{\mY^{(t)}\}$ each is Cauchy sequence, and thus converges to its critical point.
\end{enumerate}
\end{theorem}

We detail the full proof in Appendix~\ref{sec:app:convergenceProof} for the sake of consciousness. In practice, the ADMM typically defines an inner iteration number, such as $10$, in advance. This pre-defined value provides a viable fast convergence for ADMM.



\subsection{Ablation Denoising Models}
In addition to the proposed all-in-one \textsc{DoT} denoising optimization, we separate the objective into several sub-tasks and investigate the effectiveness of these individual designs. We also compare with the TV regularizer that is conventionally used for signal denoising as well as some graph node denoising tasks \cite{dong2017sparse,chen2021bigcn}. The corresponding objective function and the update rules for the three ablation models are detailed below. 

\subsubsection{Node Feature Denoising}
We define node feature denoising model similar to \eqref{func:objective_node_init}, and reformulated the problem with
$\mQ=\boldsymbol{\gW}\mU$:
\begin{equation} \label{func:objective_node_full}
    \min_{\mU} \|\vnu\mQ\|_{1,G}+\frac12\|\mU-\mX\|^2_{2,G}, \quad\text{s.t.}\quad \mQ = \boldsymbol{\gW}\mU.
\end{equation}
The function design is similar to that of \cite{dong2017sparse}, except that we optimize the objective with ADMM instead of the Split-Bregman \cite{goldstein2009split} method. The associated augmented Lagrangian problem reads
\begin{align}
    &\min_{\mU} \sum_{k,l}\nu_{k,l}\|\mQ_{k,l}\|_{1,G}+ \frac12\|\mU-\mX\|^2_{2,G} \\
    &+\sum_{k,l}\frac{\mu_{2}}{2}\|\mQ_{k,l}-\gW_{k,l}\mU\|_2^2  +\sum_{k,l}\text{tr}\left(\Lambda_{2;k,l}^{\top}(\mQ_{k,l}-\gW_{k,l}\mU)\right). \notag 
\end{align}

The above objective function can solved with the following ADMM scheme:
\begin{enumerate}[leftmargin=*]
    \item $\mU^{(t+1)}$:
    \begin{align*}
        =\left(\mD+\mu_{2}^{(t)}\mI\right)^{-1}  
        \left(\mD\mX+\sum_{k,l}\gW_{k,l}^{\top}(\mu_{2}^{(t)}\mQ^{(t)}_{k,l}+\Lambda_{2;k,l}^{(t)})\right);
    \end{align*}
    \item $\mQ_{k,l}^{(t+1)}[i,:]$:
    \begin{align*}
        =\mathcal{T}_{\nu_{k,l}d_i/\mu_{2}^{(t)}}\left(\gW_{k,l}\mU^{(t+1)}[i,:]-\frac1{\mu^{(t)}_{2}}\Lambda_{2;k,l}^{(t)}[i,:]\right);
    \end{align*}
    \item \textit{Update $\Lambda_{2;k,l}^{(t+1)}$ and $\mu_{2}^{(t+1)}$:}
    \begin{align*}
        \Lambda_{2;k,l}^{(t+1)}=&\Lambda_{2;k,l}^{(t)}+\mu_{2}^{(t)}(\mQ_{k,l}^{(t+1)}-\gW_{k,l}\mU^{(t+1)}); \\
        \mu_{2}^{(t+1)}=&\min\left(\rho\mu_{2}^{(t)}, \mu_{2,\max}\right).
    \end{align*}
\end{enumerate}

In step one of updating $\mU^{(t+1)}$, as the matrix $(\mD+\mu_{2}^{(t)}\mI)$ is diagonal, its inverse can be calculated swiftly without any approximation operation. Also, the update rule in step two is identical to that of \eqref{func:objective_1_q}, and the same batch operation can be applied here.

\subsubsection{Edge Connection Denoising}
We next consider a hybrid $\sL_{2,1}$ regularization in the case of noisy edge connections, and it is again a subset of the full objective function \eqref{func:objective_full}:
\begin{align} \label{func:objective_edge_full}
    \min_{\mU,\mZ,\mE}& \|\mZ\|_{1}+\lambda_1\|\mE\|_{2,1,G}+\frac{\lambda_2}{2}\|\mU-\mX\|^2_{2,G},\\
    s.t.&~~\text{diag}(\mZ) = 0,\;  \mU = \mZ\mU+\mE, \;\mZ \vone = \vone. \notag
\end{align}
With $\tL=\mI-\mY$, its augmented Lagrangian reads 
\begin{align}
    \gL(\mU)=&\frac12\|\mU-\mX\|^2_{2,G}+\frac{\mu_1}{2}\Bigl\|\tL\mU -\mE\Bigr\|_2^2\notag \\ 
    &+\text{tr}\left(\Lambda_1^{\top}(\tL\mU-\mE)\right).
\end{align}
By ${\partial \gL}/{\partial \mU}=0$, the update rule for $\mU^{(t+1)}$ reads
\begin{align*}
    \mU^{(t+1)}=&
    \left(\lambda_2\mD+\mu_1^{(t)}\tL^{(t)\top}\tL^{(t)}\right)^{-1}  \\
    &\left(\lambda_2\mD\mF+\mu^{(t)}_1\tL^{(t)\top}\mE^{(t)}-\Lambda_1^{(t)}\tL^{(t)}\right)\\
    \approx& \left(\frac1{\lambda_2}\mD^{-1}-\frac1{\lambda^2_2}\mu_1^{(t)}\mD^{-1}\tL^{(t)\top}\tL^{(t)}\mD^{-1}\right)\\
    & \left(\lambda_2\mD\mF+\mu^{(t)}_1\tL^{(t)\top}\mE^{(t)}-\tL^{(t)\top}\Lambda_1^{(t)}\right).
\end{align*}
Similar to before, the matrix inverse can be approximated by a Linear or Cholesky solver for computational simplicity.

\subsubsection{TV Regularizer for Node Feature}
Aside from acting the $\sL_1$-norm on Framelet coefficients in \eqref{func:objective_node_full}, the regularizer can be applied directly on the noisy raw signal with the TV regularizer for node denoising. The TV regularizer is conventionally used for signal denoising as well as some graph node denoising tasks \cite{dong2017sparse,chen2021bigcn}. Formally, the objective function is defined as
\begin{align} \label{func:objective_tv_full}
    \min_{\mU} \frac{\alpha}{2}{\rm tr}(\mU^{\top}\mL\mU)+\frac12\|\mU-\mF\|^2_{2,G}.
\end{align}
where $\alpha$ is a tunable hyper-parameter, and $\mL$ denotes the (unnormalized) graph Laplacian.

The update rule of $\mU^{(t+1)}$ of the above function is rather straight-forward. By ${\partial \sL}/{\partial \mU}=0$, we have
\begin{align*}
    \mU^{(t+1)}=(\mD+\alpha\mL)^{-1}\mD\mX \approx (\mI-\alpha\mD^{-1}\mL)\mX.
\end{align*}

\section{Numerical Experiments}
\label{sec:exp}
This section reports the performance of our proposed method in comparison to vanilla graph convolutional layers as well as three ablation models. The main experiment tests on three node classification tasks with different types of noise (feature, structure, hybrid) testified. All experiments run with PyTorch on NVIDIA\textsuperscript{\textregistered} Tesla V100 GPU with 5,120 CUDA cores and 16GB HBM2 mounted on an HPC cluster. Experimental code in PyTorch can be found at \texttt{https://github.com/bzho3923/GNN\_DoT}.

\subsection{Experimental Protocol}
We validate in this experiment the robustness of proposed methods against feature and/or structure perturbations. We poison the graph before the training process with a black-box scheme, which requires no extra training details.

\subsubsection{Dataset and Baseline}
We conduct experiments on three benchmark networks with undirected edges: 
\begin{itemize}
    \item \textbf{Cora} \cite{sen2008collective,yang2016revisiting}: a text classification dataset with $1,433$ bag-of-words representation of $2,708$ documents, and they are connected with $5,429$ citation links (edges). The target is to classify each document into one of $7$ topics. 
    \item \textbf{Citeseer} \cite{sen2008collective,yang2016revisiting}: another citation network with similar constructions as \textbf{Cora}, but contains $3,327$ nodes with $3,703$-dimensional input features and $4,732$ edges. The number of class is $6$.
    \item \textbf{Wiki-CS} \cite{mernyei2020wiki}: a snapshot of Wikipedia database from August 2019 with $216,123$ among $11,701$ articles of $10$ topics corresponding to branches of computer science. Each article is expressed by $300$ features from GloVe word embeddings. 
\end{itemize}
For all three datasets, the training set constitutes $20$ random labeled samples from each class. The validation and test data are $500$ and $1,000$, respectively.

The architecture for all models is fixed to two convolution layers filling a denoising module that overlays multiple denoising layers. We adopt one of the three graph convolutions:
\begin{itemize}
    \item \textsc{GCN} \cite{Kipf2017semi}: graph convolutional network, a $\sL_2$-based global smoothing model that is widely used for semi-supervised node classification tasks. 
    \item \textsc{GAT} \cite{velivckovic2017graph}: graph attention network, a message-passing neural network with self-attention aggregator.
    \item \textsc{UFG} \cite{zheng2021framelets}: a spectral graph neural network enhanced by fast undecimated Framelet transform. The framework uses conventional ReLU activation (\textsc{UFG-R}) for graph signal processing but can be extra robust to graph perturbation with shrinkage activation (\textsc{UFG-S}).
\end{itemize}
The denoising module falls into one of the following choices: no-denoising layer, \textsc{DoT} layers, and one of the three ablation denoising layers (node-ADMM, edge-ADMM, and node-TV) that we introduced in Section~\ref{sec:ADMMdenoising}.

\begin{table*}[t]
    \caption{Denoising performance against hybrid error. The accuracy (in percentage) is reported with standard deviation from 10 repetitive runs. }
    \label{tab:50_hybrid}
    \begin{center}
    \resizebox{0.9\textwidth}{!}{
    \begin{tabular}{llcccccc}
    \toprule
    Model & Dataset & Noise-free & Noisy & \textsc{DoT} & node-ADMM & edge-ADMM & node-TV  \\ \midrule
    \multicolumn{1}{l|}{\multirow{3}{*}{\textsc{GCN}}} & Cora & $82.23${\scriptsize$\pm0.42$} & \cellcolor{gray!25}$54.31${\scriptsize$\pm0.94$} &
    \bm{$64.98${\scriptsize$\pm1.10$}}&
    $63.84${\scriptsize$\pm1.61$}&
    $53.24${\scriptsize$\pm0.75$}&
    $54.58${\scriptsize$\pm1.35$}\\
    \multicolumn{1}{l|}{} & Citeseer & $72.30${\scriptsize$\pm0.93$} & \cellcolor{gray!25}$34.24${\scriptsize$\pm2.13$} &
    \bm{$48.16${\scriptsize$\pm0.99$}}&
    $45.06${\scriptsize$\pm1.22$}&
    $35.22${\scriptsize$\pm1.91$}&
    $34.71${\scriptsize$\pm1.21$}\\
    \multicolumn{1}{l|}{} & Wiki-CS & $78.82${\scriptsize$\pm0.38$} & \cellcolor{gray!25}$62.27${\scriptsize$\pm0.61$} &
    \bm{$64.98${\scriptsize$\pm0.31$}}&
    $64.69${\scriptsize$\pm0.75$}&
    $62.60${\scriptsize$\pm0.18$}&
    $63.22${\scriptsize$\pm0.08$}\\ \midrule
    \multicolumn{1}{l|}{\multirow{3}{*}{\textsc{GAT}}} & Cora & $82.28${\scriptsize$\pm0.53$} & \cellcolor{gray!25}$53.25${\scriptsize$\pm2.88$} & \bm{$64.45${\scriptsize$\pm1.78$}} & $64.31${\scriptsize$\pm1.47$} & $55.51${\scriptsize$\pm1.49$} & $53.38${\scriptsize$\pm3.58$} \\
    \multicolumn{1}{l|}{} & Citeseer & $71.26${\scriptsize$\pm0.60$} & \cellcolor{gray!25}$36.17${\scriptsize$\pm2.14$} &
    \bm{$46.22${\scriptsize$\pm1.23$}}&
    $45.49${\scriptsize$\pm1.47$}&
    $33.94${\scriptsize$\pm2.54$}&
    $35.41${\scriptsize$\pm2.64$}\\
    \multicolumn{1}{l|}{} & Wiki-CS & $77.58${\scriptsize $\pm0.12$} & \cellcolor{gray!25}$59.22${\scriptsize$\pm1.62$} &
    \bm{$61.52${\scriptsize$\pm0.65$}}&
    $61.25${\scriptsize$\pm0.91$}&
    $61.07${\scriptsize$\pm0.56$}&
    $60.15${\scriptsize$\pm0.47$}\\ \midrule
    \multicolumn{1}{l|}{\multirow{3}{*}{\textsc{UFG-R}}} & Cora & $83.59${\scriptsize $\pm0.63$} & \cellcolor{gray!25}$51.83${\scriptsize$\pm2.23$} &
    \bm{$63.86${\scriptsize$\pm0.75$}}&
    $61.34${\scriptsize$\pm0.95$}&
    $49.28${\scriptsize$\pm1.05$}&
    $56.87${\scriptsize$\pm2.05$}\\
    \multicolumn{1}{l|}{} & Citeseer & $72.52${\scriptsize $\pm0.44$} & \cellcolor{gray!25}$34.83${\scriptsize$\pm1.89$} &
    \bm{$45.57${\scriptsize$\pm0.99$}}&
    $45.44${\scriptsize$\pm0.46$}&
    $34.83${\scriptsize$\pm0.70$}&
    $35.86${\scriptsize$\pm1.40$}\\
    \multicolumn{1}{l|}{} & Wiki-CS & $75.09${\scriptsize $\pm0.43$} & \cellcolor{gray!25}$57.06${\scriptsize$\pm0.81$}&
    \bm{$58.13${\scriptsize$\pm0.56$}}&
    $57.81${\scriptsize$\pm0.54$}&
    $57.67${\scriptsize$\pm0.57$}&
    $57.64${\scriptsize$\pm0.57$}\\ \midrule
    \multicolumn{1}{l|}{\multirow{3}{*}{\textsc{UFG-S}}} & Cora & $83.59${\scriptsize $\pm0.63$} & \cellcolor{gray!25}$55.23${\scriptsize$\pm2.87$} &
    \bm{$61.86${\scriptsize$\pm1.17$}}&
    $58.48${\scriptsize$\pm1.05$}&
    $55.84${\scriptsize$\pm2.39$}&
    $56.54${\scriptsize$\pm3.81$}\\
    \multicolumn{1}{l|}{} & Citeseer & $72.52${\scriptsize $\pm0.44$} & \cellcolor{gray!25}$32.18${\scriptsize$\pm3.46$} &
    \bm{$45.06${\scriptsize$\pm1.51$}}&
    $43.06${\scriptsize$\pm1.47$}&
    $34.58${\scriptsize$\pm0.84$}&
    $34.59${\scriptsize$\pm0.88$}\\
    \multicolumn{1}{l|}{} & Wiki-CS & $75.09${\scriptsize $\pm0.43$} & \cellcolor{gray!25}$56.57${\scriptsize$\pm0.51$}&
    $56.71${\scriptsize$\pm0.64$}&
    $56.65${\scriptsize$\pm0.84$}&
    $56.92${\scriptsize$\pm0.68$}&
    \bm{$57.31${\scriptsize$\pm0.16$}}\\ 
    \bottomrule\\[-2.5mm]
    \end{tabular}
    }
    \end{center}
\end{table*}

\subsubsection{Training Setup}
We take data perturbation in pre-processing. As values of embedded features are binary for the two citation networks and decimal for \textbf{Wiki-CS}, we spread $25\%$ poison with random binary noise for the former datasets and Gaussian white noise with 0.25 standard deviation for the latter one. The structure perturbation on the three binary adjacency matrices is unified to a $25\%$ noise ratio with an unchanged number of edge connections. That is, we firstly randomly delete $12.5\%$ and then create $14.3\%$ of total links to keep the number of perturbed edge connection same as original graph.

For model training, we fix the number of the hidden neuron at $16$ and the dropout ratio at $0.5$ for all the models, with \textsc{Adam} optimizer. 
The network architecture for baseline models is set to two graph convolutional layers with an activation function and a dropout layer in between. The choice of the activation function follows the suggestions provided by their authors, where $ReLU$ is applied for \textsc{GCN} and \textsc{UFGConv-R}, $eLU$ is used for \textsc{GAT}, and the shrinkage activation is selected for \textsc{UFGConv-S}. The denoising layers are addressed before the second convolutional layer following another dropout layer. We removed the activation operation after the first convolutional layer, considering the denoising layers play the role of filtering latent features, and it is similar to an non-linear activation operation. All methods activate the output with a softmax for multi-class classification tasks.  
The number of epochs equals $200$. The key hyperparameters are fine-tuned with grid search. In particular, we search the learning rate and ($L_2$) weight decay from $\{0.0005,0.001,0.005,0.01,0.05\}$
and the dilation of UFG from $\{2,3,4\}$ for the base convolution. When training with the denoising layers, we fix the optimal learning rate, weight decay, dilation and $\mu_3,\mu_4$ ${1}$, and tune $\mu_1,\mu_2$ from $\{1,3,5,7,9\}$ ), and $\nu_0$ from $\{5, 10, 50, 100, 500\}$. Other unspecified hyper-parameters, such as the maximum number of level $L$ in \textsc{UFG} and the number of head in \textsc{GAT}, follows the default values suggested by PyTorch Geometric.

\begin{figure}[t]
    \centering
  \subfloat[\label{1a}]{%
       \includegraphics[width=0.48\linewidth]{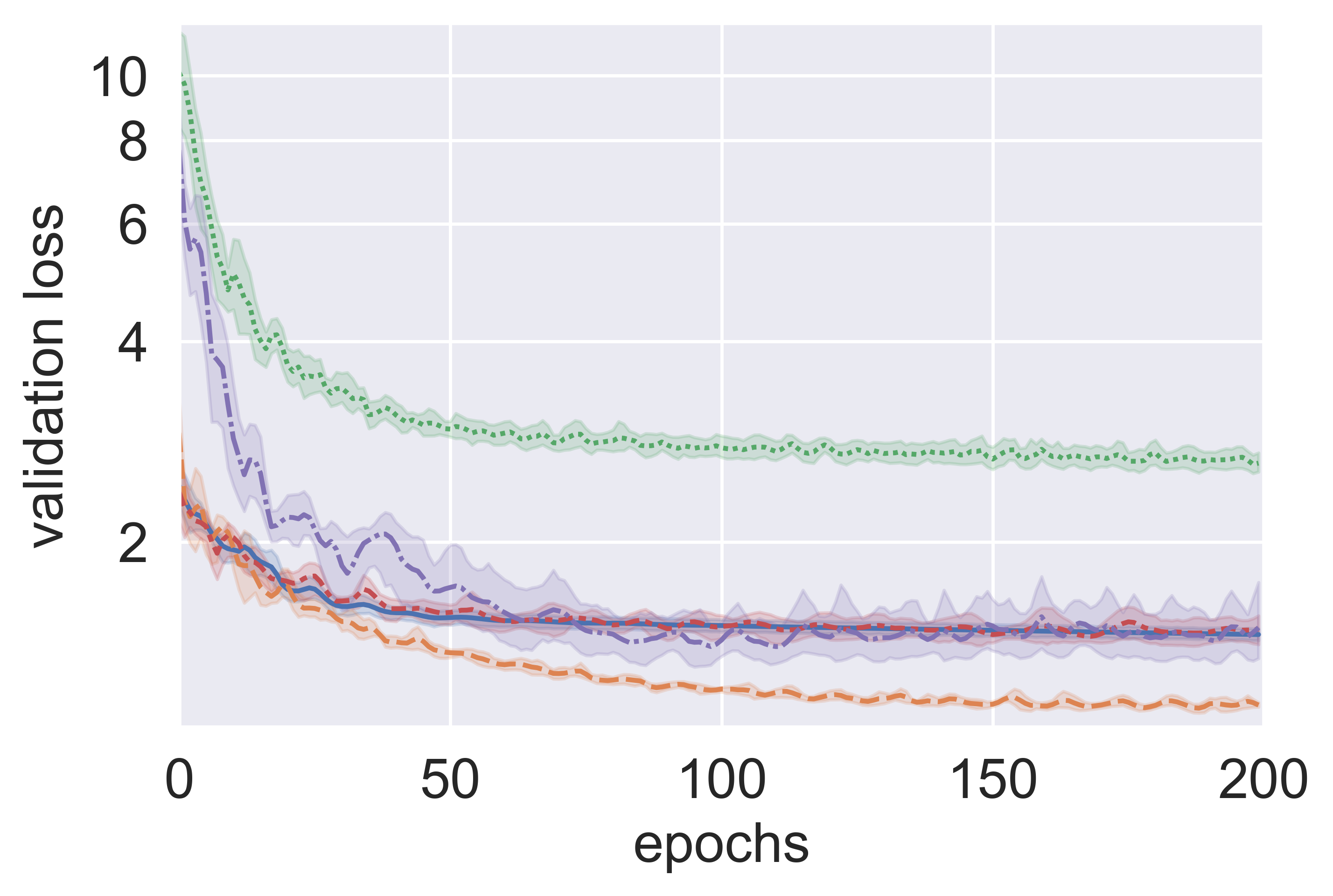}}
  \subfloat[\label{1b}]{%
        \includegraphics[width=0.48\linewidth]{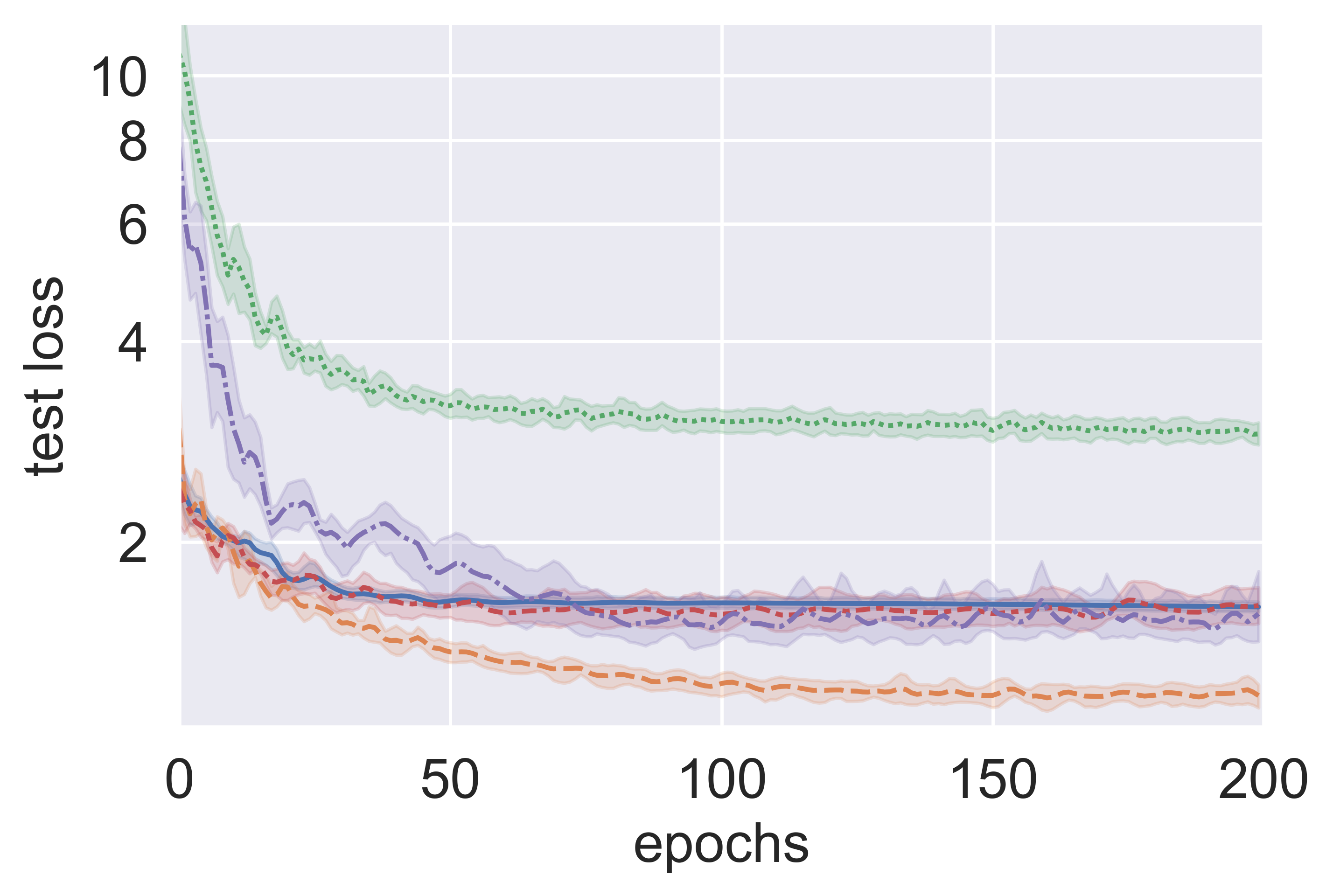}}
    \\
  \subfloat[\label{1c}]{%
        \includegraphics[width=0.48\linewidth]{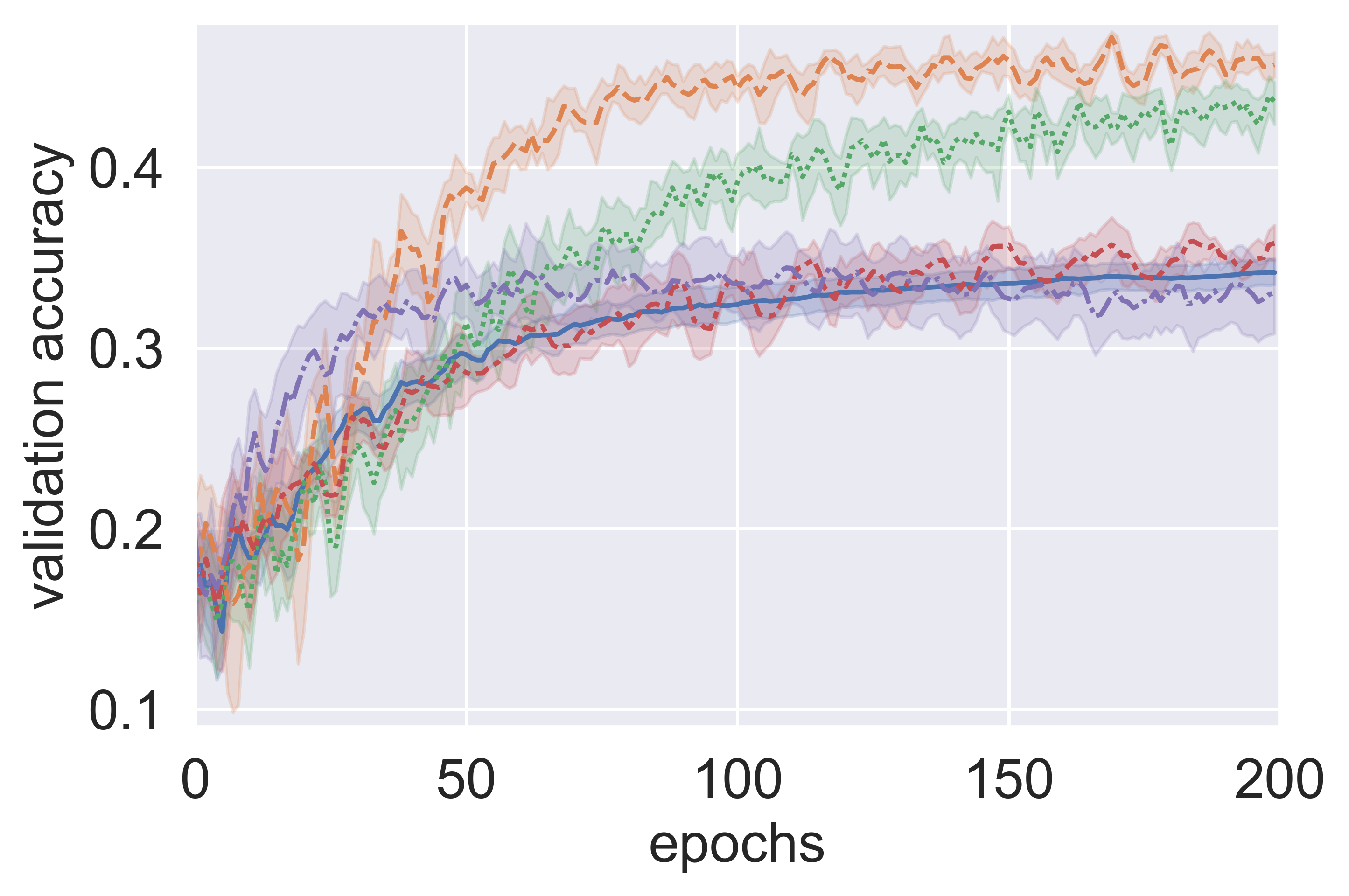}}
  \subfloat[\label{1d}]{%
        \includegraphics[width=0.48\linewidth]{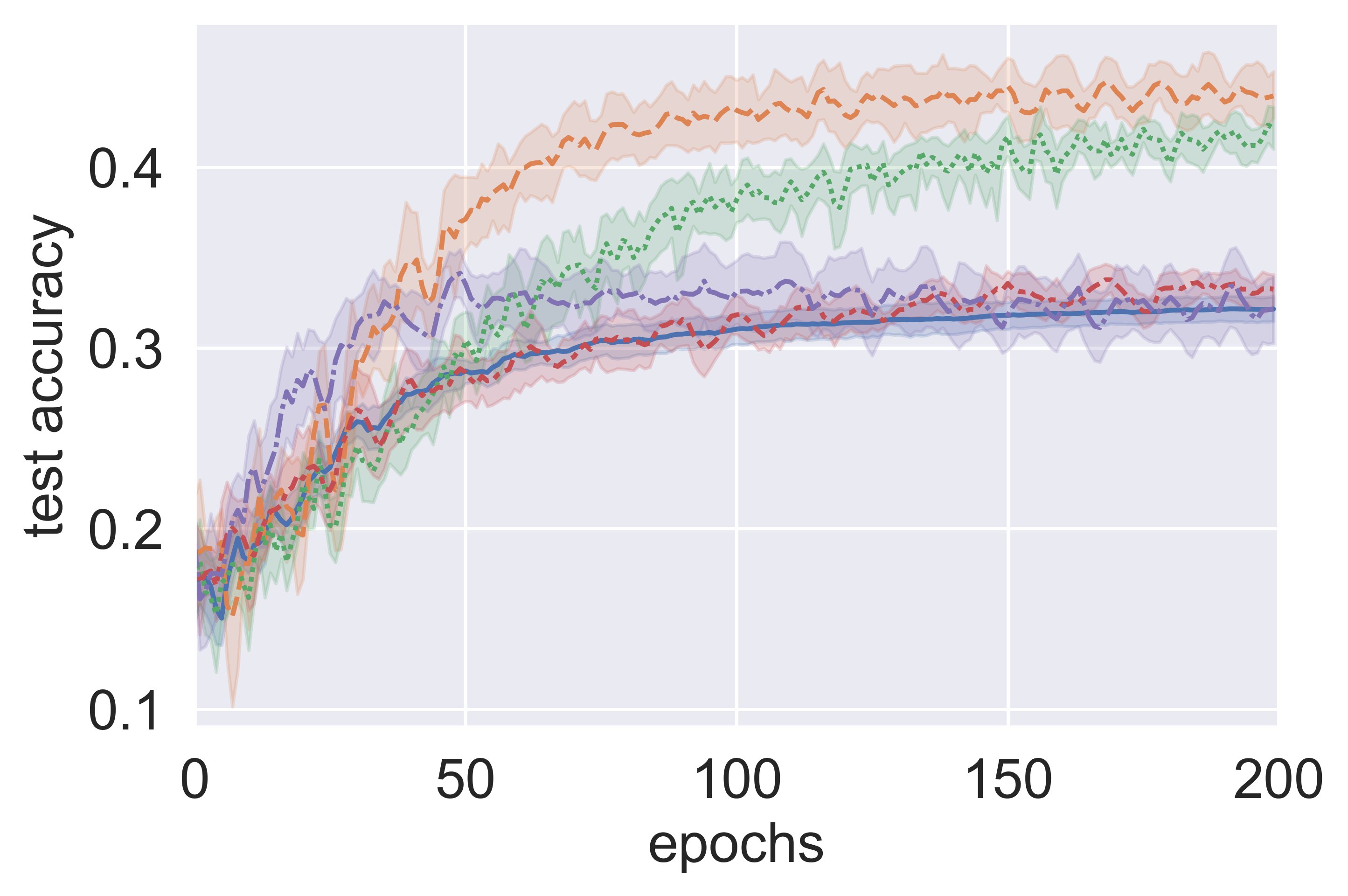}}
    \\
  \subfloat{
        \includegraphics[trim={0 3.8cm 0 3.85cm}, clip, width=0.95\linewidth]{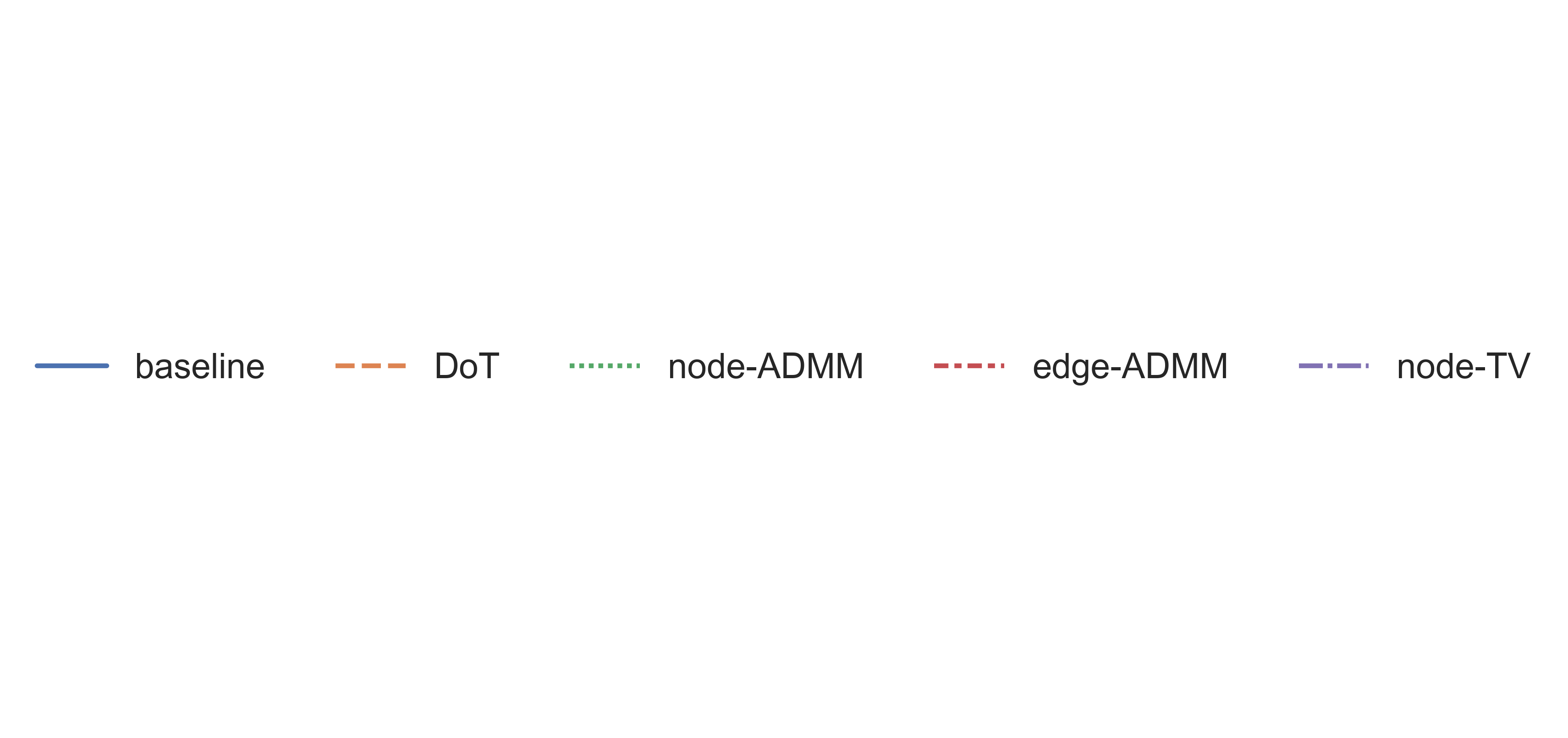}}
  \caption{Learning curve with \textsc{UFG-S} convolution on \textbf{Citeseer}, hybrid noise. We report both loss (top row) and accuracy (bottom row) curves on validation (left column) and test (right column) sets.}
  \label{fig:learningCurve} 
\end{figure}

\subsection{Results Analysis}
\subsubsection{Case 1: Hybrid Perturbation}
We report in Table~\ref{tab:50_hybrid} the average accuracy of all methods over $10$ repetitive runs. Our \textsc{DoT} outperforms ablation methods in most of the scenarios and improves up to $41\%$ against the baseline scores (in gray). The only exception is when denoising \textbf{Wiki-CS} with \textsc{UFG-S}, where \textsc{DoT} is beaten by the TV regularizer with a slightly higher accuracy. 

The superior performance of \textsc{DoT} is mainly driven by the effective node denoising module. In most cases as reported in the table, the node-ADMM method achieves the second-highest scores while the performance does not distinct from \textsc{DoT}'s. In opposite, the effect of edge denoising is less significant. Most scores in the edge-ADMM column have minor improvements over the baselines, not mention some of them actually have a hurtful of accuracy because of the absence of node-ADMM module. One possible explanation of this phenomenon is that the edge connections of graphs in a graph convolution merely directly influence the expressivity of graph embedding, especially when no targeted hostile attack is experienced. Our experiment designs non-targeted data pollution in advance, so that after a graph convolutional layer, the structure noise is implicitly reflected by feature noises. Consequently, edge noise no longer has the way to dominate the model performance, and it is sufficient to denoising on hidden embeddings from the last convolutional layer, which collects both type of noises.

While the quantitative measurements proves the promising denoising power of our proposed \textsc{DoT}, we test its robust learning process in Figure~\ref{fig:learningCurve}. In particular, we present the learning curve of \textsc{DoT} as well as its competitors' on \textbf{Citeseer}. Compared to the baseline and its ablations, \textsc{DoT} (in orange) converges rapidly and reaches stationary points at a lower level on the two loss curves. The variance are also thinner, which demonstrates a stable and reliable learning process of the trained model. 

\begin{figure*}[t]
    \centering
    \subfloat{%
          \includegraphics[width=0.3\linewidth]{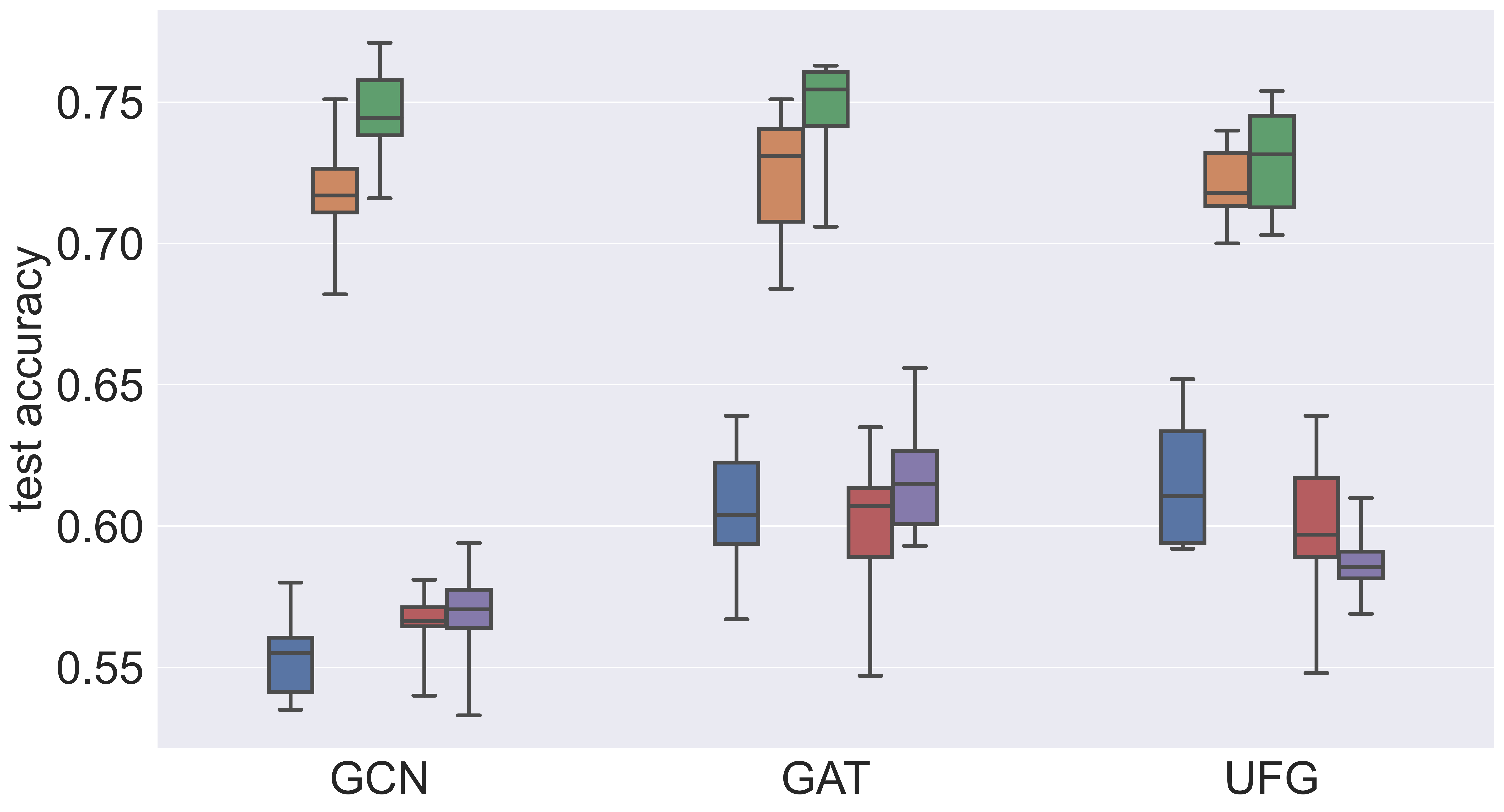}}
      \hfill
    \subfloat{%
          \includegraphics[width=0.3\linewidth]{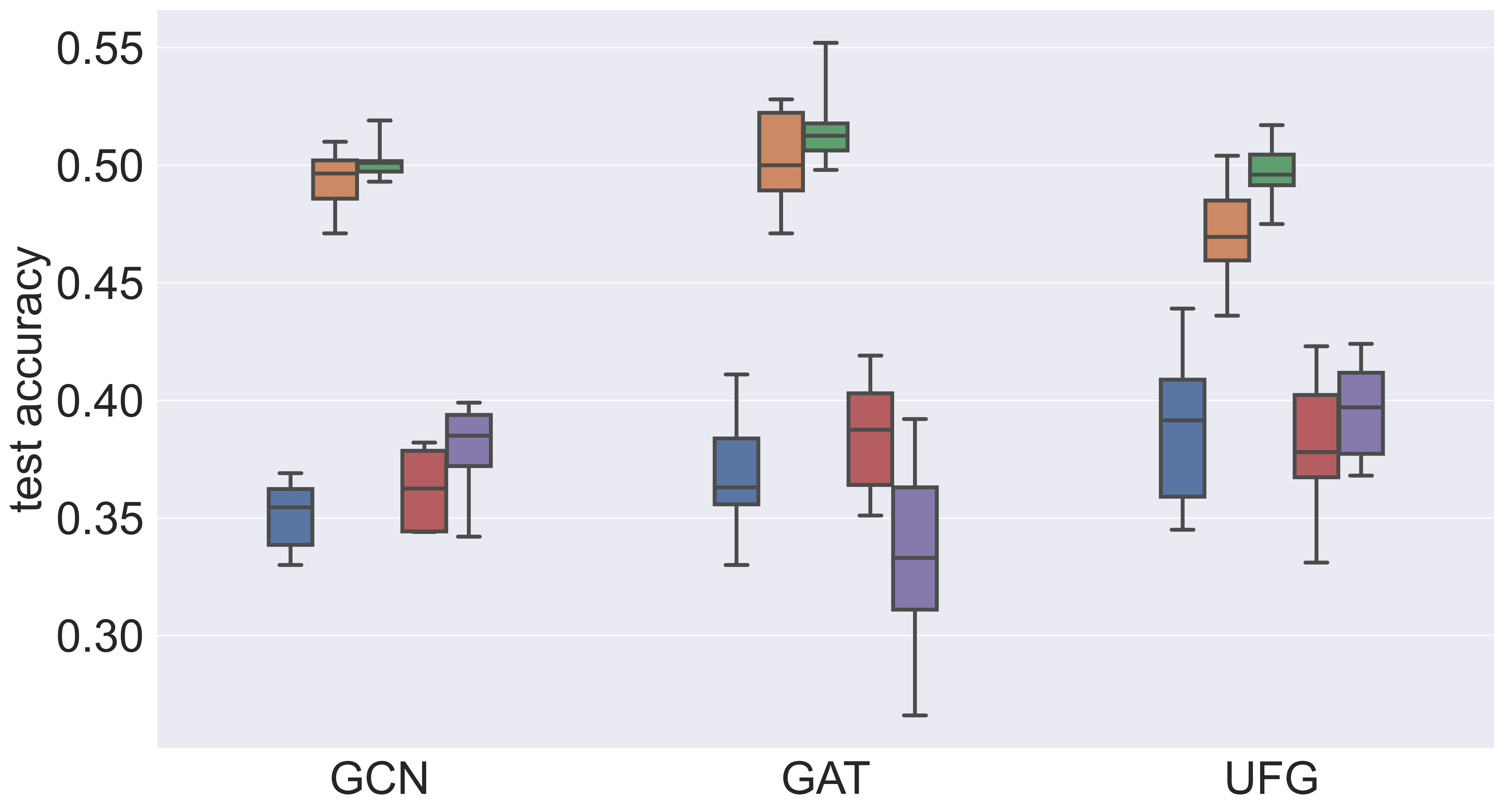}}
      \hfill
    \subfloat{%
          \includegraphics[width=0.3\linewidth]{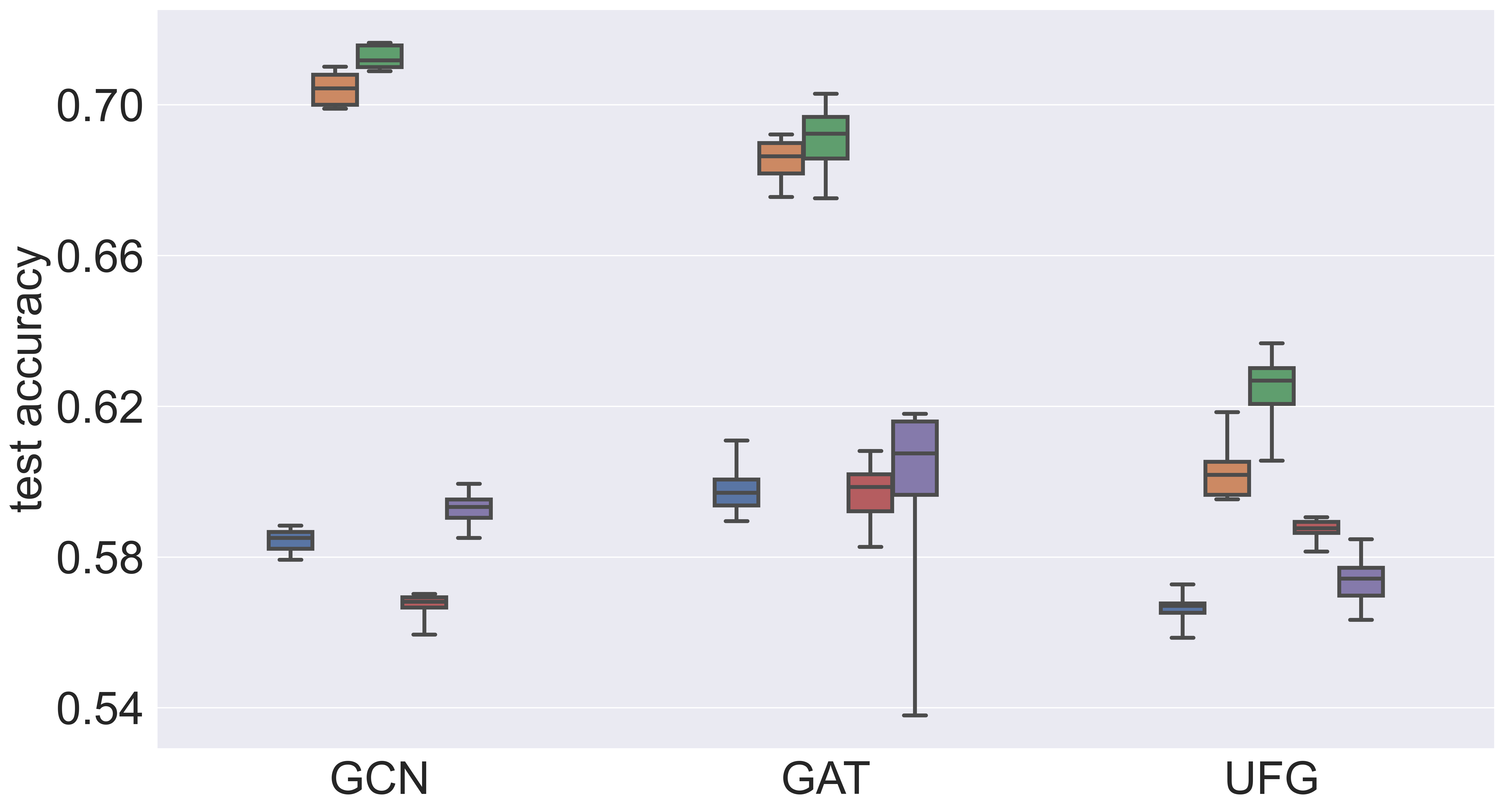}}
      \\
    \subfloat{%
          \includegraphics[width=0.3\linewidth]{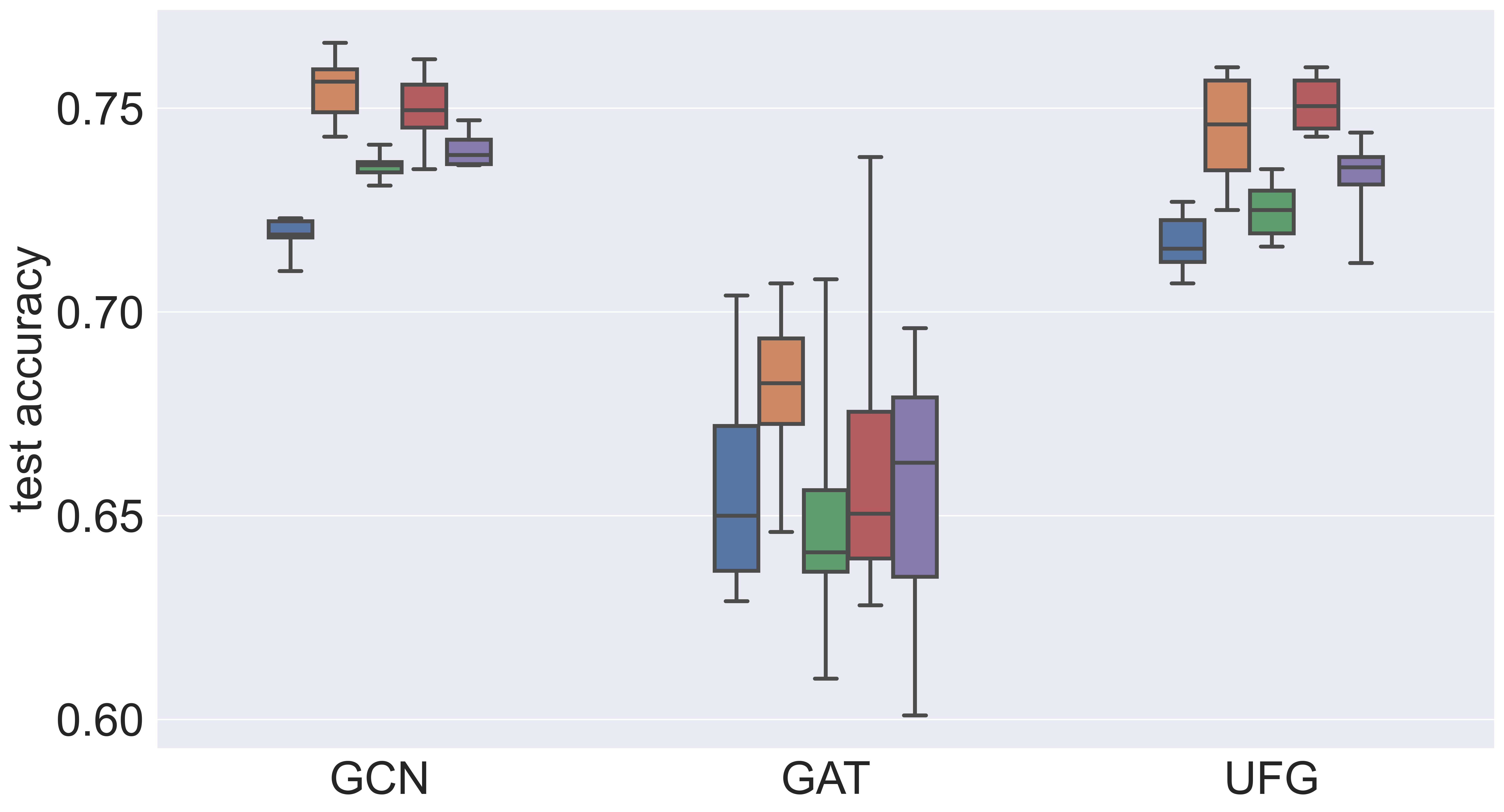}}
      \hfill
    \subfloat{%
          \includegraphics[width=0.3\linewidth]{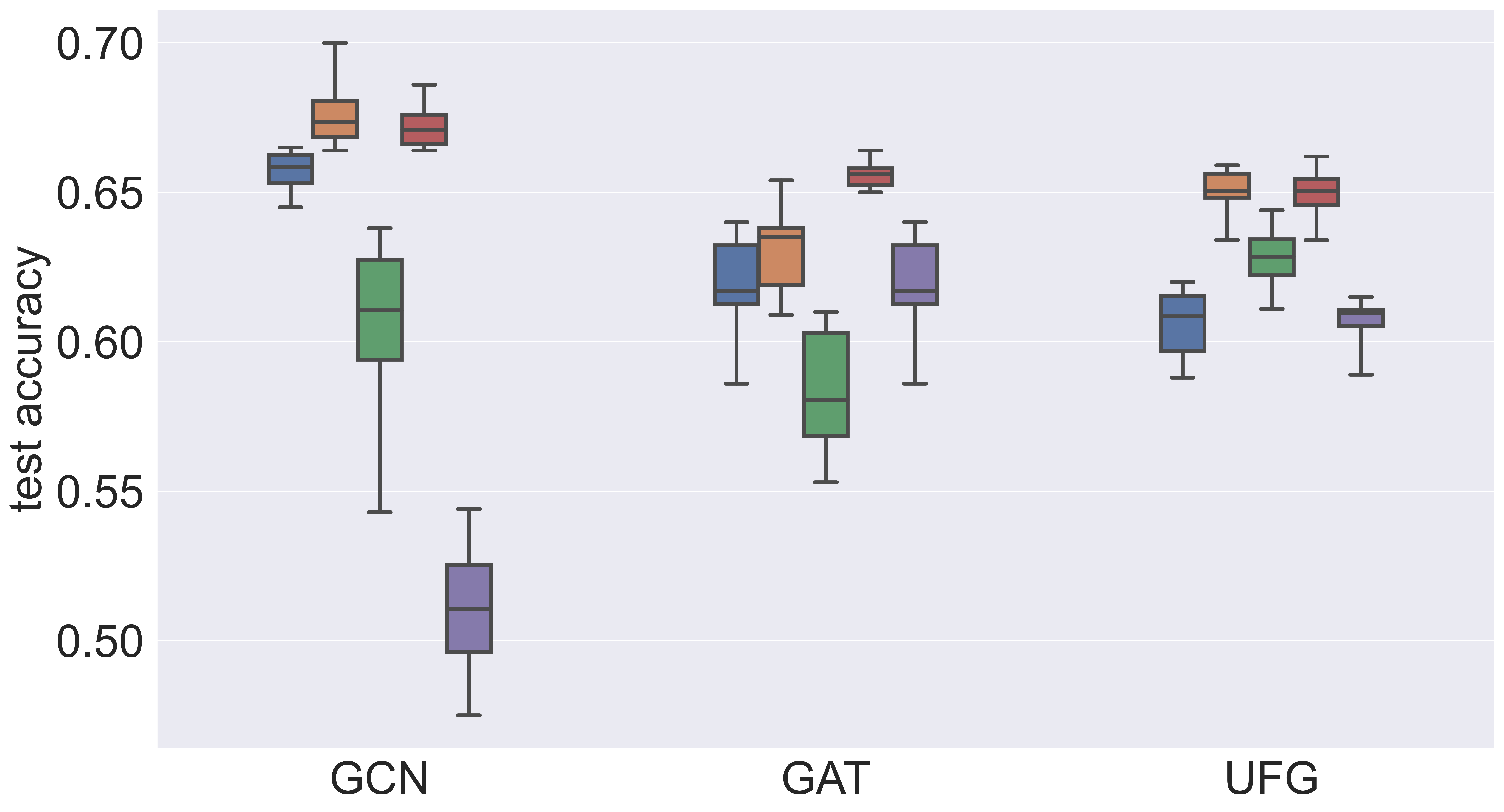}}
      \hfill
    \subfloat{%
          \includegraphics[width=0.3\linewidth]{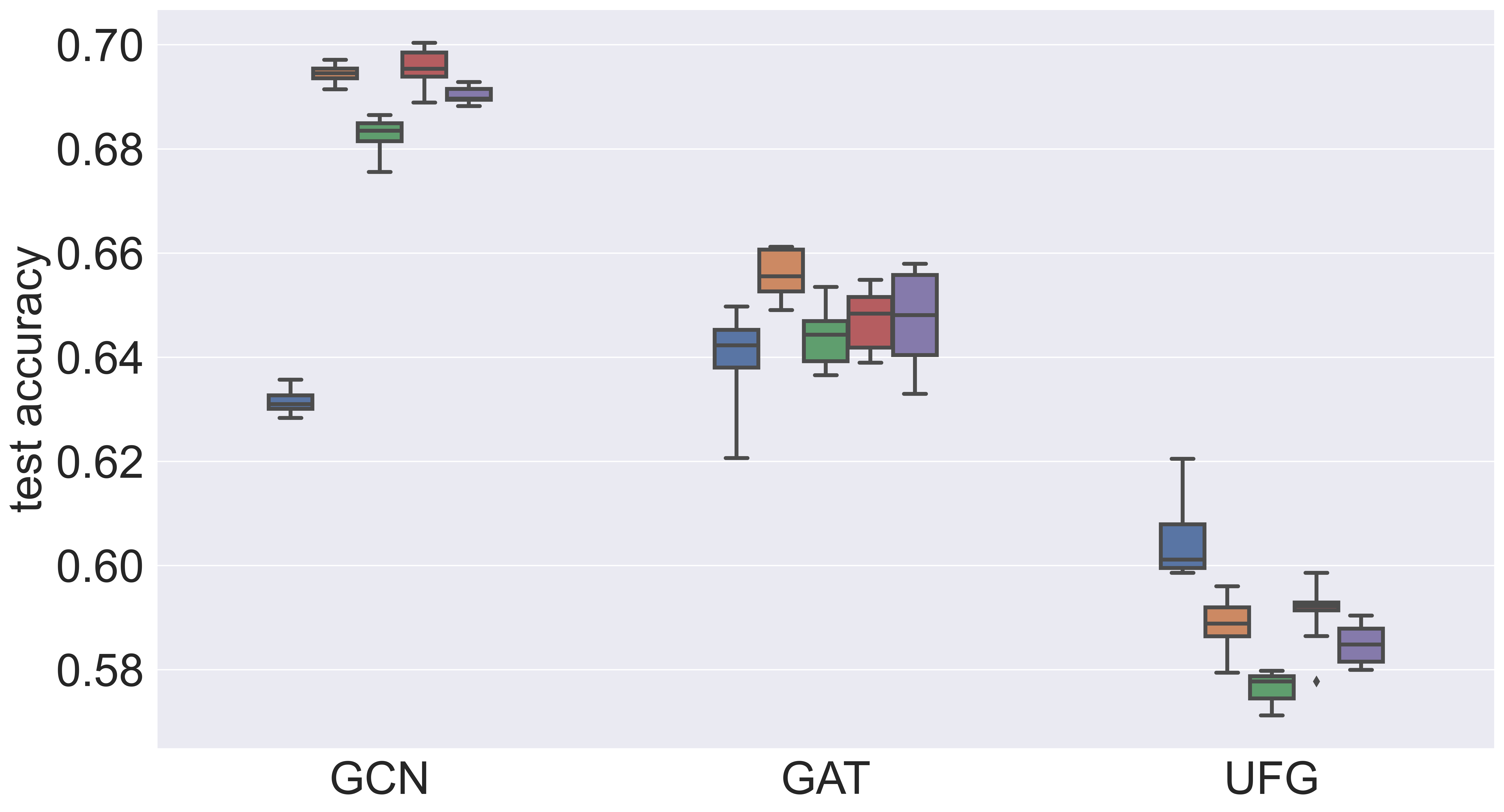}}
      \\
    \subfloat{
          \includegraphics[trim={0 3.8cm 0 3.85cm}, clip,     width=0.5\linewidth]{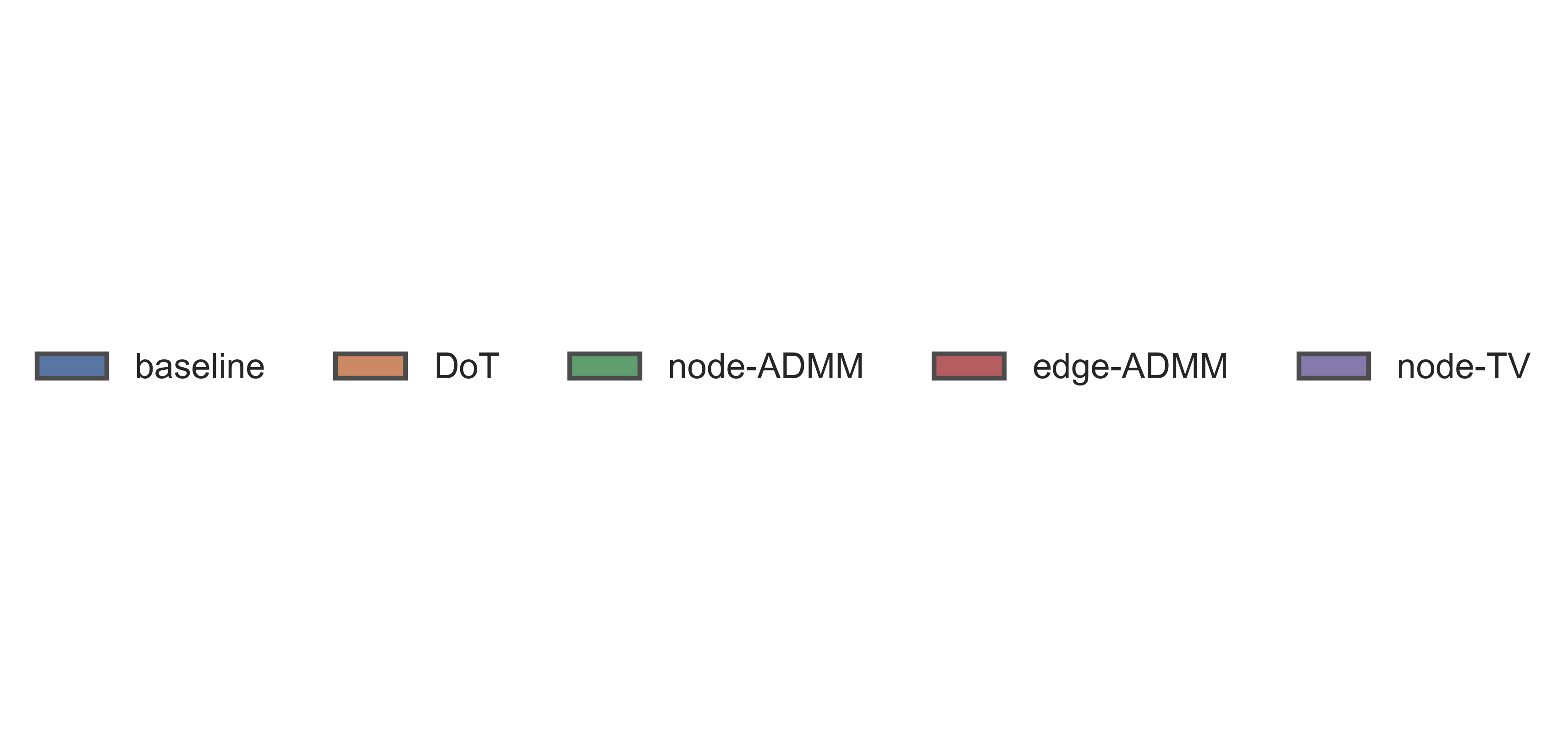}}
    \caption{Box plots of the model performance over node (row $1$) or edge (row $2$) perturbations. The plots in the three columns reports the average performance over the three datasets, \textbf{Cora} (column $1$), \textbf{Citeseer} (column $2$) and \textbf{Wiki-CS} (column $3$), respectively. }
    \label{fig:boxplot}
\end{figure*}

\subsubsection{Case 2: Individual Perturbation}
we supplement model performance under extreme feature \underline{OR} structure perturbation and investigate the main contribution of each individual denoising module. The total noise ratio is increased to a high level of $50\%$ so as to magnify the discrepancy of models. For both node and edge noises, the same distribution that was introduced previously is applied here again. For the comparison baseline models, we remove \textsc{UFG-S} with shrinkage activation, as its smoothing power over large noise has been proven in \cite{zheng2021framelets}. Appending a denoising module on top of it thus becomes less necessary.

We present the model comparison in Figure~\ref{fig:boxplot}. The results suggest exact contributors of different noise types. In the first settings of $50\%$ node noises (the first row), node-ADMM outperforms other methods significantly, following \textsc{DoT} that mainly obliterating node noises. The disparity comes from the edge-denoising module of \textsc{DoT} that is designed to minimize the norm of variance of graph topology information. When there exists little structure noise in the raw input, the extra penalty term could make a negative contribution to the overall performance in the slightest way. The edge-ADMM method barely maintains a similar performance to the baseline methods, which indicates neither of them identifies the real pains of input. The node-TV method performs surprisingly badly, though. It could be caused by the trace regularizer in the objective function design, which does not eliminate or relieve the long-existing problem of over-smoothing, like many other graph convolutional layers.

In the case of edge-denoising tasks, \textsc{DoT} and edge-ADMM perform as well as expectation. The node-ADMM and node-TV methods, although the same strategy of purifying hidden node representations has proven its effectiveness under a small level of perturbation, their performance is worse than the two methods that employ explicit edge denoising modules. This observation illustrates the importance of designing a specific edge denoising module, especially in extreme scenarios. Whenever it is difficult to justify whether the structure noise is too large to be implicitly removed, it is always suggested to include an edge module.


\section{Conclusion}
\label{sec:conclusion}
This paper introduces a robust optimization scheme for denoising graph features and structure. The graph noise level is measured by a double-term regularizer, which pushes Framelet coefficients of the node representation to be sparse and the adjacency matrix to concentrate its uncertainty on small communities with limited influence. The optimal graph representation is reached by a sequence of feed-forward propagation, where each layer functions as an optimization iteration. The update converges to the exact solution of robust graph representation, which removes irregular graph volatility but keeps expressivity. We provide comprehensive theoretical support and extensive empirical validation to prove the effectiveness of our denoising method in terms of node and/or edge contamination.

We shall emphasize that in experiments we fix the ADMM iteration up to $10$ to save the computational cost, which is lower than the typical number of $30$ to $40$. In other words, our model preserves the potential to achieve higher scores.



\ifCLASSOPTIONcaptionsoff
  \newpage
\fi



\bibliographystyle{IEEEtran}
\bibliography{reference}

\begin{thebibliography}{10}
\providecommand{\url}[1]{#1}
\csname url@samestyle\endcsname
\providecommand{\newblock}{\relax}
\providecommand{\bibinfo}[2]{#2}
\providecommand{\BIBentrySTDinterwordspacing}{\spaceskip=0pt\relax}
\providecommand{\BIBentryALTinterwordstretchfactor}{4}
\providecommand{\BIBentryALTinterwordspacing}{\spaceskip=\fontdimen2\font plus
\BIBentryALTinterwordstretchfactor\fontdimen3\font minus
  \fontdimen4\font\relax}
\providecommand{\BIBforeignlanguage}[2]{{%
\expandafter\ifx\csname l@#1\endcsname\relax
\typeout{** WARNING: IEEEtran.bst: No hyphenation pattern has been}%
\typeout{** loaded for the language `#1'. Using the pattern for}%
\typeout{** the default language instead.}%
\else
\language=\csname l@#1\endcsname
\fi
#2}}
\providecommand{\BIBdecl}{\relax}
\BIBdecl

\bibitem{westerlund2007testing}
J.~Westerlund, ``Testing for error correction in panel data,'' \emph{Oxford
  Bulletin of Economics and Statistics}, vol.~69, no.~6, pp. 709--748, 2007.

\bibitem{to2009wavelet}
A.~C. To, J.~R. Moore, and S.~D. Glaser, ``Wavelet denoising techniques with
  applications to experimental geophysical data,'' \emph{Signal Processing},
  vol.~89, no.~2, pp. 144--160, 2009.

\bibitem{goyal2020image}
B.~Goyal, A.~Dogra, S.~Agrawal, B.~S. Sohi, and A.~Sharma, ``Image denoising
  review: from classical to state-of-the-art approaches,'' \emph{Information
  Fusion}, vol.~55, pp. 220--244, 2020.

\bibitem{jin2020graph}
W.~Jin, Y.~Ma, X.~Liu, X.~Tang, S.~Wang, and J.~Tang, ``Graph structure
  learning for robust graph neural networks,'' in \emph{ACM SIGKDD}, 2020, pp.
  66--74.

\bibitem{zhu2021interpreting}
M.~Zhu, X.~Wang, C.~Shi, H.~Ji, and P.~Cui, ``Interpreting and unifying graph
  neural networks with an optimization framework,'' in \emph{Proceedings of the
  Web Conference 2021}, 2021, pp. 1215--1226.

\bibitem{Kipf2017semi}
T.~N. Kipf and M.~Welling, ``Semi-supervised classification with graph
  convolutional networks,'' in \emph{ICLR}, 2017.

\bibitem{velivckovic2017graph}
P.~Veli{\v{c}}kovi{\'c}, G.~Cucurull, A.~Casanova, A.~Romero, P.~Lio, and
  Y.~Bengio, ``Graph attention networks,'' in \emph{ICLR}, 2018.

\bibitem{xu2018graph}
B.~Xu, H.~Shen, Q.~Cao, Y.~Qiu, and X.~Cheng, ``Graph wavelet neural network,''
  in \emph{ICLR}, 2018.

\bibitem{zheng2021framelets}
X.~Zheng, B.~Zhou, J.~Gao, Y.~G. Wang, P.~Li\`{o}, M.~Li, and G.~Mont\'{u}far,
  ``How framelets enhance graph neural networks,'' in \emph{ICML}, 2021.

\bibitem{klicpera2018predict}
J.~Klicpera, A.~Bojchevski, and S.~G{\"u}nnemann, ``Predict then propagate:
  graph neural networks meet personalized pagerank,'' in \emph{ICLR}, 2018.

\bibitem{xu2018representation}
K.~Xu, C.~Li, Y.~Tian, T.~Sonobe, K.-i. Kawarabayashi, and S.~Jegelka,
  ``Representation learning on graphs with jumping knowledge networks,'' in
  \emph{ICML}, 2018.

\bibitem{bronstein2017geometric}
M.~M. {Bronstein}, J.~{Bruna}, Y.~{LeCun}, A.~{Szlam}, and P.~{Vandergheynst},
  ``Geometric deep learning: going beyond {Euclidean} data,'' \emph{IEEE Signal
  Processing Magazine}, vol.~34, no.~4, pp. 18--42, July 2017.

\bibitem{zhou2020graph}
J.~Zhou, G.~Cui, S.~Hu, Z.~Zhang, C.~Yang, Z.~Liu, L.~Wang, C.~Li, and M.~Sun,
  ``Graph neural networks: a review of methods and applications,'' \emph{AI
  Open}, vol.~1, pp. 57--81, 2020.

\bibitem{zhang2020deep}
Z.~Zhang, P.~Cui, and W.~Zhu, ``Deep learning on graphs: a survey,'' \emph{IEEE
  Transactions on Knowledge and Data Engineering}, 2020.

\bibitem{wu2020comprehensive}
Z.~Wu, S.~Pan, F.~Chen, G.~Long, C.~Zhang, and S.~Y. Philip, ``A comprehensive
  survey on graph neural networks,'' \emph{IEEE Transactions on Neural Networks
  and Learning Systems}, vol.~32, no.~1, pp. 4--24, 2020.

\bibitem{xu2018powerful}
K.~Xu, W.~Hu, J.~Leskovec, and S.~Jegelka, ``How powerful are graph neural
  networks?'' in \emph{ICLR}, 2018.

\bibitem{jiang2021co}
X.~Jiang, R.~Zhu, S.~Li, and P.~Ji, ``Co-embedding of nodes and edges with
  graph neural networks,'' \emph{IEEE Transactions on Pattern Analysis and
  Machine Intelligence}, pp. 1--1, 2020.

\bibitem{chen2021graph}
S.~Chen, Y.~C. Eldar, and L.~Zhao, ``Graph unrolling networks: interpretable
  neural networks for graph signal denoising,'' \emph{IEEE Transactions on
  Signal Processing}, 2021.

\bibitem{nt2019revisiting}
H.~Nt and T.~Maehara, ``Revisiting graph neural networks: all we have is
  low-pass filters,'' \emph{arXiv:1905.09550}, 2019.

\bibitem{bruna2013spectral}
J.~Bruna, W.~Zaremba, A.~Szlam, and Y.~LeCun, ``Spectral networks and locally
  connected networks on graphs,'' in \emph{ICLR}, 2014.

\bibitem{henaff2015deep}
M.~Henaff, J.~Bruna, and Y.~LeCun, ``Deep convolutional networks on
  graph-structured data,'' \emph{arXiv:1506.05163}, 2015.

\bibitem{defferrard2016convolutional}
M.~Defferrard, X.~Bresson, and P.~Vandergheynst, ``Convolutional neural
  networks on graphs with fast localized spectral filtering,'' in \emph{NIPS},
  2016.

\bibitem{levie2018cayleynets}
R.~Levie, F.~Monti, X.~Bresson, and M.~M. Bronstein, ``Cayleynets: graph
  convolutional neural networks with complex rational spectral filters,''
  \emph{IEEE Transactions on Signal Processing}, vol.~67, no.~1, pp. 97--109,
  2018.

\bibitem{shuman2013emerging}
D.~I. Shuman, S.~K. Narang, P.~Frossard, A.~Ortega, and P.~Vandergheynst, ``The
  emerging field of signal processing on graphs: extending high-dimensional
  data analysis to networks and other irregular domains,'' \emph{IEEE Signal
  Processing Magazine}, vol.~30, no.~3, pp. 83--98, 2013.

\bibitem{zheng2020mathnet}
X.~Zheng, B.~Zhou, M.~Li, Y.~G. Wang, and J.~Gao, ``{MathNet: Haar-like}
  wavelet multiresolution-analysis for graph representation and learning,''
  \emph{arXiv:2007.11202}, 2020.

\bibitem{tremblay2018design}
N.~Tremblay, P.~Gon{\c{c}}alves, and P.~Borgnat, ``Design of graph filters and
  filterbanks,'' in \emph{Cooperative and Graph Signal Processing}.\hskip 1em
  plus 0.5em minus 0.4em\relax Elsevier, 2018, pp. 299--324.

\bibitem{liao2018lanczosnet}
R.~Liao, Z.~Zhao, R.~Urtasun, and R.~Zemel, ``Lanczosnet: multi-scale deep
  graph convolutional networks,'' in \emph{ICLR}, 2019.

\bibitem{isufi2016autoregressive}
E.~Isufi, A.~Loukas, A.~Simonetto, and G.~Leus, ``Autoregressive moving average
  graph filtering,'' \emph{IEEE Transactions on Signal Processing}, vol.~65,
  no.~2, pp. 274--288, 2016.

\bibitem{thomas2021higher}
T.~Schnake, O.~Eberle, J.~Lederer, S.~Nakajima, K.~T. Schutt, K.-R. Mueller,
  and G.~Montavon, ``Higher-order explanations of graph neural networks via
  relevant walks,'' \emph{IEEE Transactions on Pattern Analysis and Machine
  Intelligence}, pp. 1--1, 2021.

\bibitem{zhou2004learning}
D.~Zhou, O.~Bousquet, T.~N. Lal, J.~Weston, and B.~Sch{\"o}lkopf, ``Learning
  with local and global consistency,'' in \emph{NIPS}, 2004.

\bibitem{fu2020understanding}
G.~Fu, Y.~Hou, J.~Zhang, K.~Ma, B.~F. Kamhoua, and J.~Cheng, ``Understanding
  graph neural networks from graph signal denoising perspectives,''
  \emph{arXiv:2006.04386}, 2020.

\bibitem{vaseghi2008advanced}
S.~V. Vaseghi, \emph{Advanced Digital Signal Processing and Noise
  Reduction}.\hskip 1em plus 0.5em minus 0.4em\relax John Wiley \& Sons, 2008.

\bibitem{parrilli2011nonlocal}
S.~Parrilli, M.~Poderico, C.~V. Angelino, and L.~Verdoliva, ``A nonlocal sar
  image denoising algorithm based on llmmse wavelet shrinkage,'' \emph{IEEE
  Transactions on Geoscience and Remote Sensing}, vol.~50, no.~2, pp. 606--616,
  2011.

\bibitem{bahoura2006wavelet}
M.~Bahoura and J.~Rouat, ``Wavelet speech enhancement based on time--scale
  adaptation,'' \emph{Speech Communication}, vol.~48, no.~12, pp. 1620--1637,
  2006.

\bibitem{loizou2007speech}
P.~C. Loizou, \emph{Speech Enhancement: Theory and Practice}.\hskip 1em plus
  0.5em minus 0.4em\relax CRC press, 2007.

\bibitem{chang2000adaptive}
S.~G. Chang, B.~Yu, and M.~Vetterli, ``Adaptive wavelet thresholding for image
  denoising and compression,'' \emph{IEEE Transactions on Image Processing},
  vol.~9, no.~9, pp. 1532--1546, 2000.

\bibitem{aminghafari2006multivariate}
M.~Aminghafari, N.~Cheze, and J.-M. Poggi, ``Multivariate denoising using
  wavelets and principal component analysis,'' \emph{Computational Statistics
  \& Data Analysis}, vol.~50, no.~9, pp. 2381--2398, 2006.

\bibitem{ramani2008monte}
S.~Ramani, T.~Blu, and M.~Unser, ``Monte-carlo sure: a black-box optimization
  of regularization parameters for general denoising algorithms,'' \emph{IEEE
  Transactions on Image Processing}, vol.~17, no.~9, pp. 1540--1554, 2008.

\bibitem{ding2015artifact}
Y.~Ding and I.~W. Selesnick, ``Artifact-free wavelet denoising: non-convex
  sparse regularization, convex optimization,'' \emph{IEEE Signal Processing
  Letters}, vol.~22, no.~9, pp. 1364--1368, 2015.

\bibitem{liu2021elastic}
X.~Liu, W.~Jin, Y.~Ma, Y.~Li, H.~Liu, Y.~Wang, M.~Yan, and J.~Tang, ``Elastic
  graph neural networks,'' in \emph{ICML}, 2021.

\bibitem{chen2021bigcn}
Z.~Chen, T.~Ma, Z.~Jin, Y.~Song, and Y.~Wang, ``{BiGCN}: a bi-directional
  low-pass filtering graph neural network,'' \emph{arXiv:2101.05519}, 2021.

\bibitem{dong2017sparse}
B.~Dong, ``Sparse representation on graphs by tight wavelet frames and
  applications,'' \emph{Applied and Computational Harmonic Analysis}, vol.~42,
  no.~3, pp. 452--479, 2017.

\bibitem{gabay1976dual}
D.~Gabay and B.~Mercier, ``A dual algorithm for the solution of nonlinear
  variational problems via finite element approximation,'' \emph{Computers \&
  Mathematics with Applications}, vol.~2, no.~1, pp. 17--40, 1976.

\bibitem{goldstein2009split}
T.~Goldstein and S.~Osher, ``The split bregman method for l1-regularized
  problems,'' \emph{SIAM journal on imaging sciences}, vol.~2, no.~2, pp.
  323--343, 2009.

\bibitem{xie2019differentiable}
X.~Xie, J.~Wu, G.~Liu, Z.~Zhong, and Z.~Lin, ``Differentiable linearized
  {ADMM},'' in \emph{ICML}.\hskip 1em plus 0.5em minus 0.4em\relax PMLR, 2019,
  pp. 6902--6911.

\bibitem{li2020training}
J.~Li, M.~Xiao, C.~Fang, Y.~Dai, C.~Xu, and Z.~Lin, ``Training neural networks
  by lifted proximal operator machines,'' \emph{IEEE Transactions on Pattern
  Analysis and Machine Intelligence}, 2020.

\bibitem{wang2015trend}
Y.-X. Wang, J.~Sharpnack, A.~Smola, and R.~Tibshirani, ``Trend filtering on
  graphs,'' in \emph{Artificial Intelligence and Statistics}.\hskip 1em plus
  0.5em minus 0.4em\relax PMLR, 2015, pp. 1042--1050.

\bibitem{Gilmer_etal2017}
J.~Gilmer, S.~S. Schoenholz, P.~F. Riley, O.~Vinyals, and G.~E. Dahl, ``Neural
  message passing for quantum chemistry,'' in \emph{ICML}, 2017.

\bibitem{wu2019simplifying}
F.~Wu, A.~Souza, T.~Zhang, C.~Fifty, T.~Yu, and K.~Weinberger, ``Simplifying
  graph convolutional networks,'' in \emph{ICML}, 2019.

\bibitem{xu2019graph}
B.~Xu, H.~Shen, Q.~Cao, Y.~Qiu, and X.~Cheng, ``Graph wavelet neural network,''
  in \emph{ICLR}, 2019.

\bibitem{ortega2018graph}
A.~Ortega, P.~Frossard, J.~Kova{\v{c}}evi{\'c}, J.~M. Moura, and
  P.~Vandergheynst, ``Graph signal processing: overview, challenges, and
  applications,'' \emph{Proceedings of the IEEE}, vol. 106, no.~5, pp.
  808--828, 2018.

\bibitem{cai2010split}
J.-F. Cai, S.~Osher, and Z.~Shen, ``Split bregman methods and frame based image
  restoration,'' \emph{Multiscale Modeling \& Simulation}, vol.~8, no.~2, pp.
  337--369, 2010.

\bibitem{nocedal2006numerical}
J.~Nocedal and S.~Wright, \emph{Numerical Optimization}.\hskip 1em plus 0.5em
  minus 0.4em\relax Springer Science \& Business Media, 2006.

\bibitem{deng2017parallel}
W.~Deng, M.-J. Lai, Z.~Peng, and W.~Yin, ``Parallel multi-block admm with o
  (1/k) convergence,'' \emph{Journal of Scientific Computing}, vol.~71, no.~2,
  pp. 712--736, 2017.

\bibitem{Golub1996matrix}
G.~H. Golub, \emph{\BIBforeignlanguage{eng}{Matrix Computations}}, 3rd~ed.,
  ser. Johns Hopkins studies in the mathematical sciences.\hskip 1em plus 0.5em
  minus 0.4em\relax Baltimore: Johns Hopkins University Press, 1996.

\bibitem{Haddad2009Cholesky}
C.~N. Haddad, \emph{Encyclopedia of Optimization}.\hskip 1em plus 0.5em minus
  0.4em\relax Springer US, 2009, ch. {C}holesky {F}actorization.

\bibitem{sen2008collective}
P.~Sen, G.~Namata, M.~Bilgic, L.~Getoor, B.~Galligher, and T.~Eliassi-Rad,
  ``Collective classification in network data,'' \emph{AI Magazine}, vol.~29,
  no.~3, pp. 93--93, 2008.

\bibitem{yang2016revisiting}
Z.~Yang, W.~Cohen, and R.~Salakhudinov, ``Revisiting semi-supervised learning
  with graph embeddings,'' in \emph{ICML}, 2016.

\bibitem{mernyei2020wiki}
P.~Mernyei and C.~Cangea, ``Wiki-cs: a wikipedia-based benchmark for graph
  neural networks,'' \emph{arXiv:2007.02901}, 2020.

\bibitem{miller1981inverse}
K.~S. Miller, ``On the inverse of the sum of matrices,'' \emph{Mathematics
  Magazine}, vol.~54, no.~2, pp. 67--72, 1981.

\bibitem{woodbury1950inverting}
M.~A. Woodbury, ``Inverting modified matrices,'' \emph{Memorandum Report},
  vol.~42, no. 106, p. 336, 1950.

\bibitem{lin2011linearized}
Z.~Lin, R.~Liu, and Z.~Su, ``Linearized alternating direction method with
  adaptive penalty for low-rank representation,'' in \emph{NIPS}, 2011.

\end{thebibliography}

%




\newpage
\appendices

\section{DoT Update Scheme in Detail}
\label{sec:app:admmUpdate}
This section derives the calculation details of the update rules for \textsc{DoT} in Section~\ref{sec:ADMMdenoising}. 

\subsection[STEP 1: U]{Update $\mU^{(t+1)}$}
We start from updating $\mU^{(t+1)}$ with $\mZ^{(t)}, \mE^{(t)}, \mQ_{k,l}^{(t)}$, $\Lambda_{k,l}^{(t)}$ and $\mu^{(t)}_{k,l}$s. We omit the constant term in \eqref{func:objective_full_lagrangian} and rewrite the minimization function to
\begin{align*}
    &\gL(\mU)=\frac{\lambda_2}2\|\mU-\mX\|^2_{2,G}  +\frac{\mu^{(t)}_1}{2}\Bigl\|\tL^{(t)}\mU -\mE^{(t)}\Bigr\|_2^2 \\
    &\qquad\qquad+\text{tr}\left(\Lambda_1^{(t)\top}(\tL^{(t)}\mU-\mE^{(t)})\right)\\
    &\sum_{k,l}\left(\frac{\mu^{(t)}_{2}}{2}\bigl\|\mQ^{(t)}_{k,l}-\gW_{k,l}\mU\bigr\|_2^2+ \text{tr}\left(\Lambda_{2;k,l}^{(t)\top}(\mQ^{(t)}_{k,l}-\gW_{k,l}\mU)\right) \right),
\end{align*}
where $\tL^{(t)}=\mI-\mY^{(t)}$. Note that $\frac12\|\mU-\mX\|^2_{2,G} = \frac12\text{tr}((\mU-\mX)^{\top}\mD (\mU-\mX))$ where $\mD = \text{diag}(d_1, ..., d_N)^{\top}$. 
We then take ${\partial \gL}/{\partial \mU}=0$, which gives 
\begin{align*}
    &\mU^{(t+1)}=
    \left(\lambda_2\mD+\mu^{(t)}_1\tL^{(t)\top}\tL^{(t)}+\mu_{2}^{(t)}\mI\right)^{-1}\\
    &\qquad\qquad\left(\lambda_2\mD\mX +\mu^{(t)}_1\tL^{(t)\top}\mE^{(t)}-\tL^{(t)\top}\Lambda_1^{(t)}\right.\\
    &\qquad\qquad\left.+\sum_{k,l}(\mu_{2}^{(t)}\gW_{k,l}^{\top}\mQ^{(t)}_{k,l}+\gW_{k,l}^{\top}\Lambda_{2;k,l}^{(t)})\right)\\
    &\hspace{-2mm}\approx\left(\mI-\mu^{(t)}_2(\lambda_2\mD+\mu^{(t)}_2\mI)^{-1}\tL^{(t)\top}\tL^{(t)}\right)\left(\lambda_2\mD+\mu^{(t)}_2\mI\right)^{-1} \\
    &\quad\left(\lambda_2\mD\mX+\mu^{(t)}_1\tL^{(t)\top}\mE^{(t)}-\tL^{(t)\top}\Lambda_1^{(t)}\right.\\
    &\quad\left.+\sum_{k,l}(\mu_{2}^{(t)}\gW_{k,l}^{\top}\mQ^{(t)}_{k,l}+\gW_{k,l}^{\top}\Lambda_{2;k,l}^{(t)})\right).
\end{align*}
The above equation contains matrix inversion, where the Maclaurin Series of first order Taylor expansion \cite{miller1981inverse} is applied for approximation. Alternatively, one could consider the linear solver \cite{Golub1996matrix} or the Cholesky Factorization \cite{Haddad2009Cholesky}. in the case when the spectral radius of the inverse matrix is larger than one and the matrix approximation fails to converge.

\subsection[STEP 2: Z]{Update $\mZ^{(t+1)}$}
For the third step, we fix $\mU^{(t+1)}$, $\mY^{(t+1)}$ and all other variables  from iteration $t$ to update $\mZ^{(t+1)}$.  The sub-problem is to minimize the following objective
\begin{align*}
    \gL(\mZ)=& \|\mZ\|_1 +\frac{\mu^{(t)}_4}2 \|\mY^{(t)} - \mZ + \text{diag}(\mZ)\|^2 \\
    &+\text{tr}\left(\Lambda^{(t)\top}_4\left(\mY^{(t)} - \mZ +\text{diag}(\mZ)\right)\right).
\end{align*}
That can be solved by the following closed-form solution
\begin{align*}
    \mZ^{(t+1)} =& \mR - \text{diag}(\mR),\\
    \text{ where }\mR :=& \mathcal{T}_{1/\mu^{(t)}_4}\left(\mY^{(t)} + \frac1{\mu^{(t)}_4}\Lambda^{(t)}_4\right).
\end{align*}
Here the $\mathcal{T}_{\eta}(\cdot)$ is the \emph{soft threshold operator} defined as follows:
\begin{align*}
    \mathcal{T}_{\eta}(x) = \text{sign}(x)\max\{|x| - \eta, 0\}
\end{align*}
that $\text{ReLU}(x-\eta) - \text{ReLU}(-x-\eta)$.

\subsection[STEP 3: E]{Update $\mE^{(t+1)}$}
The associated objective function with respect to $\mE$ reads
\begin{align*}
    \gL(\mE)=&\|\mE\|_{2,1,G}+\frac{\mu_1^{(t)}}{2}\bigl\|\tL^{(t+1)}\mU^{(t+1)}-\mE\bigr\|_2^2\\
    &+\text{tr}\left(\Lambda_1^{(t)\top}\bigl(\tL^{(t+1)}\mU^{(t+1)}-\mE\bigr)\right)\\
    =&\|\mE\|_{2,1,G}+\frac{\mu_1^{(t)}}{2}\left\|\mE-\Bigl(\tL^{(t+1)}\mU^{(t+1)}+\frac{\Lambda_1^{(t)}}{\mu_1^{(t)}}\Bigr)\right\|_2^2,
\end{align*}
where $\tL^{(t+1)}=\mI-\mY^{(t+1)}$. The solution to the $i$th row of $\mE^{(t+1)}$ is then
\begin{align*}
    \mE_i^{(t+1)}=\mathcal{T}_{1/\mu_1^{(t)}}^i\left(\tL^{(t+1)}\mU^{(t+1)}+\frac{\Lambda_1^{(t)}}{\mu_1^{(t)}}\right).
\end{align*}
Here $\mathcal{T}^i$ is a row-wise soft-thresholding for group $\sL_2$-regularization. For the $i$th row of $x$,
\begin{align*}
\mathcal{T}_{\eta}^i(x) = \frac{x_i}{\|x_i\|_2}\max\{\|x_i\|_2-\eta,0\}.
\end{align*}

\subsection[STEP 4: Y]{Update $\mY^{(t+1)}$}
We next use fixed $\mU^{(t+1)}$ and all other variables from iteration $t$ to update $\mY^{(t+1)}$. Similar to the last step, we rewrite the optimization problem by omitting the constant terms, which gives
\begin{align*}
    \gL(\mY)=&\frac{\mu^{(t)}_1}{2}\|\mU^{(t+1)}-\mY\mU^{(t+1)}-\mE^{(t+1)}\|_2^2+\frac{\mu^{(t)}_3}{2}\|\mY\vone-\vone\|_2^2\\
    &+ \frac{\mu^{(t)}_4}{2}\|\mY - \mZ^{(t+1)} +\text{diag}(\mZ^{(t+1)})\|^2_2\\
    &+ \text{tr}\left(\Lambda_1^{(t)\top}(\mU^{(t+1)}-\mY\mU^{(t+1)}-\mE^{(t)})\right)\\
    &+\Lambda_{3}^{(t)\top}(\mY\vone-\vone) + \text{tr}\left(\Lambda_4^{(t)\top}(\mY - \mZ^{(t)} +\text{diag}(\mZ^{(t)}))\right).
\end{align*}
With ${\partial \gL}/{\partial \mY}=0$, we have 
\begin{align*}
    \mY^{(t+1)}=&\left(\mu^{(t)}_1(\mU^{(t+1)} - \mE^{(t+1)})\mU^{(t+1)\top} + \mu^{(t)}_3\vone\vone^{\top}  \right. \\
    & \left.+ \mu^{(t)}_4 \mZ^{(t+1)} + \Lambda^{(t)}_1\mU^{(t+1)\top} - \Lambda^{(t)}_3\vone^{\top} -\Lambda^{(t)}_4\right)\\
    & \left(\mu^{(t)}_1\mU^{(t+1)}\mU^{(t+1)\top} + \mu^{(t)}_3\vone\vone^{\top} +\mu^{(t)}_4\mI\right)^{-1}.
\end{align*}

Calculating the above equation involves the inversion of a considerably large matrix. When $N>>d$, we suggest apply the \emph{Woodbury identity} \cite{woodbury1950inverting} to reduce the $\R^{N\times N}$ matrix to a smaller size of $\R^{(d+1)\times (d+1)}$. We rewrite
\begin{align*}
    &\mu^{(t)}_1\mU^{(t+1)}\mU^{(t+1)\top} + \mu^{(t)}_3\vone\vone^{\top} +\mu^{(t)}_4\mI& \\
    =& \widetilde{\mU}^{(t+1)}\widetilde{\mU}^{(t+1)\top} + \mu^{(t)}_4\mI,&
\end{align*}
where $\widetilde{\mU}^{(t+1)} =[\sqrt{\mu^{(t)}_1}\mU^{(t+1)}, \sqrt{\mu^{(t)}_3}\vone] \in\mathbb{R}^{N\times (d+1)}$. By the Woodbury identity, we reduce
\begin{flalign*}
    &(\widetilde{\mU}^{(t+1)}\widetilde{\mU}^{(t+1)\top} + \mu^{(t)}_4\mI)^{-1}&\\ \approx& \frac1{\mu^{(t)}_4}\mI-\frac1{\mu^{(t)}_4} \widetilde{\mU}^{(t+1)}\left[\mu^{(t)}_4 \mI + \widetilde{\mU}^{(t+1)\top} \widetilde{\mU}^{(t+1)}\right]^{-1} \widetilde{\mU}^{(t+1)\top},&
\end{flalign*}
where $\mu^{(t)}_4 \mI + \widetilde{\mU}^{(t+1)\top} \widetilde{\mU}^{(t+1)}$ is in size $(d+1)\times (d+1)$.

\subsection[STEP 3: Q]{Update $\mQ^{(t+1)}$}
We first rewrite the objective function \eqref{func:objective_full_lagrangian} as
\begin{flalign*}
    &\gL(\mQ)=\sum_{k,l}\nu_{k,l}\|\mQ_{k,l}\|_{1,G}+\frac{\mu_{2}^{(t)}}{2}\left\|\mQ_{k,l}-\gW_{k,l}\mU^{(t+1)}\right\|_2^2\\
    &\qquad\qquad+\text{tr}\left(\Lambda_{2;k,l}^{(t)\top}\left(\mQ_{k,l}-\gW_{k,l}\mU^{(t+1)}\right)\right)\\
    &=\sum_{k,l}\nu_{k,l}\|\mQ_{k,l}\|_{1,G}+\frac{\mu_{2}^{(t)}}{2}\left\|\mQ_{k,l}-\left(\gW_{k,l}\mU^{(t+1)}-\frac{\Lambda_{2;k,l}^{(t)}}{\mu_{2}^{(t)}}\right)\right\|_2^2.
\end{flalign*}
The above formulation suggests a row-wise update of $\mQ_{k,l}^{(t+1)}$. For the $i$th row,
\begin{align*}
    \mQ_{k,l}^{(t+1)}[i,:]=\mathcal{T}_{\nu_{k,l}d_i/\mu_{2}^{(t)}}\left(\gW_{k,l}\mU^{(t+1)}[i,:]-\frac1{\mu^{(t)}_{2}}\Lambda_{2;k,l}^{(t)}[i,:]\right).
\end{align*}
\begin{remark}
    A batch operation can be considered here in implementation. Denote $d$ the feature dimension of $\mX$, and $\mD$ the diagonal of graph degrees. We define $\Delta$ as a matrix consisting of $d$ repeated columns from $\mD$ so that $\Delta_{k,l} = \frac{\nu_{k,l}}{\mu^{(t)}_{2}} \Delta$. With matrix operation we have $\mQ_{k,l}^{(t+1)} = \mathcal{T}_{\Delta_{k,l}}\left(\gW_{k,l}\mU^{(t+1)}-\frac1{\mu^{(t)}_{2}}\Lambda_{2;k,l}^{(t)}\right)$.
\end{remark}

\subsection[STEP 5: lambda]{Update $\Lambda^{(t+1)}$s}
We now update the Lagrangian multipliers with respect to the three constraints. Fix variables $\mU^{(t+1)}$, $\mQ_{k,l}^{(t+1)}, \mZ^{(t+1)}, \mE^{(t+1)}$ and the parameters $\mu^{(t)}$s, we update the multiplier $\Lambda_{k,l}^{(t+1)}$s by
\begin{align*}
    \Lambda_{1}^{(t+1)}&=\Lambda_{1}^{(t)}+\mu_{1}^{(t)}(\mU^{(t+1)} - \mY^{(t+1)}\mU^{(t+1)}-\mE^{(t+1)})\\
    \Lambda_{2;k,l}^{(t+1)}&=\Lambda_{2;k,l}^{(t)}+\mu_{2}^{(t)}(\mQ_{k,l}^{(t+1)}-\gW_{k,l}\mU^{(t+1)})\\
    \Lambda_{3}^{(t+1)}&=\Lambda_{3}^{(t)}+\mu_{3}^{(t)}(\mY^{(t+1)} \vone-\vone)
    \\
    \Lambda_{4}^{(t+1)}&=\Lambda_{4}^{(t)}+\mu_{4}^{(t)}(\mY^{(t+1)} - \mZ^{(t+1)} + \text{diag}(\mZ^{(t+1)}))
\end{align*}
As $\Lambda$s preserve the integral of all residuals with respect to the constrain along the update progress, ADMM feeds back the difference stepwisely to drive a zero error on the constrain.

\subsection[STEP 6: mu]{Update $\vmu^{(t+1)}$s}
Finally, we consider adaptive penalty parameters $\mu^{(t+1)}$s to lift the burden of parameter tuning. Such updating rule has been proven to have a fast convergence speed, see \cite{lin2011linearized}. For $i=1,2,3,4$, we define
\begin{align}
\begin{aligned}
    \mu_{i}^{(t+1)}&=\min\left(\rho\mu_{i}^{(t)}, \mu_{i,\max}\right).
 \end{aligned}\label{func:updateMu}
\end{align}
Here $\mu_{i,\max}$, for $i=1,2,3,4$, is an upper bound over all the adpative penalty parameters, and the value of $\rho\, (\geq 1)$ is fixed to an appropriate constant. 
\begin{remark}
    The $\rho$ can also be defined as a piece-wise function where the split condition comes from stopping criteria analysis, for instance, see \cite{xie2019differentiable}.
\end{remark}

\section{Proof of Theorem~\ref{theorem:convergence}}
\label{sec:app:convergenceProof}
 
Step 1: The proof of claim 1)

Consider the iteration scheme for $\mZ^{(t+1)}$. On the $(t+1)$th iteration, according to the updating rule of $\mE^{(t+1)}$, its first-order optimality condition holds, i.e., 
\[
\mathbf 0 \in \partial \|\mZ^{(t+1)}\|_1 - \mu^{(t)}_4 (\mY^{(t+1)} - \mZ^{(t+1)} +\text{diag}(\mZ^{(t+1)})) - \Lambda^{(t)}_4
\]
According to the second last rule in \eqref{func:objective_1_mu_lambda}, we immediately see that 
\begin{align}
 \Lambda^{(t+1)}_4 \in \partial \|\mZ^{(t+1)}\|_1. \label{eq:Gao1}   
\end{align} 
According to the fact, when $y\in\partial \|x\|$ where $\|\cdot\|$ is a norm, we have $\|y\|_{\text{dual}}\leq 1$. Specifically $\|y\|_{\text{dual}} = 1$ if $x\not=0$, otherwise $\|y\|_{\text{dual}} \leq 1$. As the dual norm of the $\ell_1$ norm is $\ell_{\infty}$, hence $\Lambda^{(t+1)}$ is bounded.  

Now consider each  $\mQ_{k,l}$. The first order condition of the update rule \eqref{func:objective_1_q} is given by
\[
\mathbf{0} \in \nu_{k,l}\partial \|\mQ^{(t+1)}_{k,l}\|_{1,G} + \Lambda_{2;k,l} + \mu^{(t+1)}_2(\mQ_{k,l} -\mathcal{W}_{k,l} \mU^{(t+1)}).
\]
This surely gives, as our argument above,
\begin{align}
   \Lambda^{(t+1)}_{2;k,l} \in \nu_{k,l}\partial \|\mQ^{(t+1)}_{k,l}\|_{1,G},\label{eq:Gao2} 
\end{align} 
hence all $\Lambda^{(t+1)}_{2;k,l}$ are bounded.

By the similar argument as the first step, we can claim that 
\begin{align}
 \Lambda^{(t+1)}_1 \in \partial_{\mE} \|\mE^{(t+1)}\|_{2,1,G}. \label{eq:Gao3}   
\end{align} 
Similarly we can conclude that $\Lambda^{(t+1)}_{1}$ is bounded too.

Applying the similar strategy for $\ell_{2,1}$ norm, we can conclude that $\Lambda^{(t+1)}_1$ is also bounded.

Regarding the boundedness of $\Lambda^{(t)}_3$, we note that the first order condition for the update step 1 for $\mY^{(t+1)}$ gives 
\begin{align*}
&-\mu^{(t)}_1(\mU^{(t+1)} - \mY^{(t+1)} \mU^{(t+1)} - \mE^{(t+1)})\mU^{(t+1)\top} \\
&+\mu^{(t)}_3 (\mY^{(t+1)}\mathbf 1 -\mathbf 1)\mathbf 1^\top + \mu^{(t)}_4(\mY^{(t+1)} - \mZ^{(t+1)} +\text{diag}(\mZ^{(t+1)})\\
&- \Lambda^{(t)}_1\mU^{(t+1)\top} + \Lambda^{(t)}_3\mathbf 1^\top + \Lambda^{(t)}_4 = 0,
\end{align*}
which is, based on the Lagrangian updating rules \eqref{func:objective_1_mu_lambda},
\begin{align*}
    &-(\Lambda^{(t+1)}_1 - \Lambda^{(t)}_1)\mU^{(t+1)\top} + (\Lambda^{(t+1)}_3 - \Lambda^{(t)}_3)\vone^\top \\
    &+ (\Lambda^{(t+1)}_4 - \Lambda^{(t)}_4) - \Lambda^{(t+1)}_4\mU^{(t+1)} + \Lambda^{(t)}_3\vone^\top + \Lambda^{(t)}_4 = 0.
\end{align*}
Hence
\begin{align}
 \Lambda^{(t+1)}_3 \vone^\top = -\Lambda^{(t+1)}_1\mU^{(t+1)} +  \Lambda^{(t+1)}_4 \label{eq:Gao4}
\end{align}
which shows $\Lambda^{(t+1)}_3$ is bounded.


Finally the monotonic property of ADMM gives  
\begin{align*}
&\gL(\mU^{(t+1)},\mZ^{(t+1)},\mE^{(t+1)},\mQ^{(t+1)}, \mY^{(t+1)};\Lambda^{(t)}_1,\Lambda^{(t)}_2,\\
&\;\;\;\;\;\;\Lambda^{(t)}_3,\Lambda^{(t)}_4,\mu^{(t)}_1,\mu^{(t)}_2,\mu^{(t)}_3, \mu^{(t)}_4) \\
\leq & \gL(\mU^{(t)},\mZ^{(t)},\mE^{(t)},\mQ^{(t)},\mY^{(t)};\Lambda^{(t)}_1,\Lambda^{(t)}_2,\Lambda^{(t)}_3,\Lambda^{(t)}_4,\\
&\;\;\;\;\;\;\mu^{(t)}_1,\mu^{(t)}_2,\mu^{(t)}_3, \mu^{(t)}_4). 
\end{align*}
It is easy to re-write that
\begin{align*}
 & \gL(\mU^{(t)},\mZ^{(t)},\mE^{(t)},\mQ^{(t)},\mY^{(t)};\Lambda^{(t)}_1,\Lambda^{(t)}_2,\Lambda^{(t)}_3,\Lambda^{(t)}_4,\\
 & \;\;\;\;\;\;\;\;\mu^{(t)}_1,\mu^{(t)}_2,\mu^{(t)}_3, \mu^{(t)}_4) \\
 =&  \gL(\mU^{(t)},\mZ^{(t)},\mE^{(t)},\mQ^{(t)},\mY^{(t)};\Lambda^{(t-1)}_1,\Lambda^{(t-1)}_2,\Lambda^{(t-1)}_3,\\
 &\;\;\;\;\Lambda^{(t-1)}_4,\mu^{(t-1)}_1,\mu^{(t-1)}_2,\mu^{(t-1)}_3, \mu^{(t-1)}_4)\\
 &+\frac{\mu^{(t)}_1 - \mu^{(t-1)}_1}{2}\|\mU^{(t)}-\mY^{(t)}\mU^{(t)} -\mE^{(t)}\|_2^2 \\
 &+\sum_{k,l}\frac{\mu^{(t)}_{2} - \mu^{(t-1)}_2}{2}\|\mQ^{(t)}_{k,l}-\gW_{k,l}\mU^{(t)}\|_2^2 \notag \\
    &+\frac{\mu^{(t)}_3 -\mu^{(t-1)}_3}{2}\|\mY^{(t)} \vone-\vone\|_2^2 \\
    &+ \frac{\mu^{(t)}_4-\mu^{(t-1)}_4}2\|\mY^{(t)}-\mZ^{(t)}+\text{diag}(\mZ^{(t)})\|^2_2 \notag \\ 
    & +\text{tr}\left((\Lambda_1^{(t)\top}-\Lambda_1^{(t-1)\top})(\mU^{(t)}-\mY^{(t)}\mU^{(t)}-\mE^{(t)})\right)\\
    &+\sum_{k,l}\text{tr}\left((\Lambda_{2;k,l}^{(t)\top}-\Lambda_{2;k,l}^{(t-1)\top})(\mQ_{k,l}-\gW_{k,l}\mU^{(t)})\right) \notag \\ 
    &+(\Lambda_{3}^{(t)\top}-\Lambda_{3}^{(t-1)\top})(\mY^{(t)} \vone-\vone)\\ &+\text{tr}((\Lambda^{(t)\top}_4-\Lambda^{(t-1)\top}_4)(\mY^{(t)} - \mZ^{(t)} + \text{diag}(\mZ^{(t)})))\\
=& \text{RHS}
\end{align*}
With the updating rules in Step 6 of the algorithm, we have
\begin{align*}
    \text{RHS}&=\gL(\mU^{(t)},\mZ^{(t)},\mE^{(t)},\mQ^{(t)},\mY^{(t)};\Lambda^{(t-1)}_1,\Lambda^{(t-1)}_2,\\ 
    &\;\;\;\;\;\Lambda^{(t-1)}_3,\Lambda^{(t-1)}_4,\mu^{(t-1)}_1,\mu^{(t-1)}_2,\mu^{(t-1)}_3, \mu^{(t-1)}_4)\\
    & +\frac{\mu^{(t)}_1+\mu^{(t-1)}_1}{2\mu^{2(t-1)}_1}\|\Lambda_1^{(t)}-\Lambda_1^{(t-1)}\|^2_F  \\
    &+\sum_{k,l} \frac{\mu^{(t)}_2+\mu^{(t-1)}_2}{2\mu^{2(t-1)}_2}\|\Lambda_{2;k,l}^{(t)}-\Lambda_{2;k,l}^{(t-1)}\|^2_F   \notag \\ 
    &+ \frac{\mu^{(t)}_3+\mu^{(t-1)}_3}{2\mu^{2(t-1)}_3}\|\Lambda_{3}^{(t)}-\Lambda_{3}^{(t-1)}\|^2_F\\ &+ \frac{\mu^{(t)}_4+\mu^{(t-1)}_4}{2\mu^{2(t-1)}_4}\|\Lambda^{(t)}_4-\Lambda^{(t-1)}_4\|^2_F.
\end{align*}
Summing the above result over $t= 0,1, ..., T (\geq 1)$ gives
\begin{align*}
 & \gL(\mU^{(t)},\mZ^{(t)},\mE^{(t)},\mQ^{(t)},\mY^{(t)};\Lambda^{(t)}_1,\Lambda^{(t)}_2,\Lambda^{(t)}_3,\Lambda^{(t)}_4,\\
 &\;\;\;\;\;\;\;\mu^{(t)}_1,\mu^{(t)}_2,\mu^{(t)}_3, \mu^{(t)}_4) \\
 \leq & \gL(\mU^{(1)},\mZ^{(1)},\mE^{(1)},\mQ^{(1)}, \mY^{(1)};\Lambda^{(0)}_1,\Lambda^{(0)}_2,\Lambda^{(0)}_3,\Lambda^{(0)}_4,\\
 &\;\;\;\;\mu^{(0)}_1,\mu^{(0)}_2,\mu^{(0)}_3, \mu^{(0)}_4) \\
 &+\sum^T_{t=1}\frac{\mu^{(t)}_1+\mu^{(t-1)}_1}{2\mu^{2(t-1)}_1}\|\Lambda_1^{(t)}-\Lambda_1^{(t-1)}\|^2_F  \\
    &+\sum^T_{t=1}\sum_{k,l} \frac{\mu^{(t)}_2+\mu^{(t-1)}_2}{2\mu^{2(t-1)}_2}\|\Lambda_{2;k,l}^{(t)}-\Lambda_{2;k,l}^{(t-1)}\|^2_F   \notag \\ 
    &+\sum^T_{t=1} \frac{\mu^{(t)}_3+\mu^{2(t-1)}_3}{2\mu^{2(t-1)}_3}\|\Lambda_{3}^{(t)}-\Lambda_{3}^{(t-1)}\|^2_F\\ &+ \sum^T_{t=1}\frac{\mu^{(t)}_4+\mu^{2(t-1)}_4}{2\mu^{2(t-1)}_4}\|\Lambda^{(t)}_4-\Lambda^{(t-1)}_4\|^2_F.
\end{align*}
Note that, for $i=1, 2, 3, 4$
\begin{align*}
&\sum^T_{t=1}\frac{\mu^{(t)}_i+\mu^{(t-1)}_i}{2\mu^{2(t-1)}_i} =    \sum^T_{t=1}\frac{\rho \mu^{(t-1)}_i+\mu^{(t-1)}_i}{2\mu^{2(t-1)}_i} \\
=&\frac{\rho+1}2\sum^T_{t=1}\frac1{\mu^{(t-1)}_i} = \frac{\rho+1}{2\mu^{(0)}_i}\sum^T_{t=1}\frac1{\rho^{t-1}}\leq \frac{\rho(\rho+1)}{2\mu^{(0)}_i (\rho-1)}.
\end{align*}

Note that $\gL(\mU^{(1)},\mZ^{(1)},\mE^{(1)},\mQ^{(1)},\mY^{(1)};\Lambda^{(0)}_1,\Lambda^{(0)}_2,\Lambda^{(0)}_3,$ $\Lambda^{(0)}_4,\mu^{(0)}_1,\mu^{(0)}_2,\mu^{(0)}_3, \mu^{(0)}_4)$ is finite, and squences $\{\Lambda^{(t)}_1, \Lambda^{(t)}_2, \Lambda^{(t)}_3, \Lambda^{(t)}_4\}$ and $\sum^T_{t=1}\frac{\mu^{(t)}_i+\mu^{(t-1)}_i}{2\mu^{2(t-1)}_i}$, so the objective is bounded. Now Therefore, we can conclude that  all the relevant sequences $\{\Gamma_t\}$ are unbounded given the finite objective function values. This completes the proof of claim 1) of Theorem 1.

\noindent Step 2: The proof of claim 2)

According to Bolzano-Weierstrass  theorem, it follows from the boundedness of the sequence $\{\Gamma_t\}_{t=1}^{\infty}$ that $\{\Gamma_t\}_{t=1}^{\infty}$ must have a convergent subsequence. Without loss of generality, suppose that the subsequence of $\{\Gamma_t\}_{t=1}^{\infty}$ is represented by itself, and it converges to an accumulation point, denoted as $\Gamma_*$. That is, we have established 
\begin{align}
    &\lim\limits_{t\rightarrow \infty} (\mU^{(t)}, \mZ^{(t)},  \mE^{(t)}, \mQ^{(t)}, \mY^{(t)}, \Lambda^{(t)}_1,\Lambda^{(t)}_2, \Lambda^{(t)}_3,\Lambda^{(t)}_4) \notag \\
=& (\mU_*, \mZ_*,  \mE_*, \mQ_*, \mY_*, \Lambda_{1*},\Lambda_{2*}, \Lambda_{3*},\Lambda_{4*}).\label{eq:conv}
\end{align}

The updating rules of Lagrange multipliers \eqref{func:objective_edge_full} imply that
 \begin{align*} 
&\mU^{(t+1)} - \mZ^{(t+1)}\mU^{(t+1)}-\mE^{(t+1)} =  (\Lambda_{1}^{(t+1)} - \Lambda_{1}^{(t)})/ \mu_{1}^{(t)}; \notag \\
&\mQ_{k,l}^{(t+1)}-  \gW_{k,l}\mU^{(t+1)}=(\Lambda_{2;k,l}^{(t+1)}-\Lambda_{2;k,l}^{(t)})/\mu_{2}^{(t)}; \\
& \mY^{(t+1)} \vone-\vone =       (\Lambda_{3}^{(t+1)}-\Lambda_{3}^{(t)})/\mu_{3}^{(t)}; \notag  \\
&\mY^{(t+1)} - \mZ^{(t+1)} + \text{diag}(\mZ^{(t+1)}) =         (\Lambda_{4}^{(t+1)}-\Lambda_{4}^{(t)})/\mu_{4}^{(t)};
\end{align*}
By the boundedness of the sequences $\{ \Lambda^{(t)}_1,\Lambda^{(t)}_2, \Lambda^{(t)}_3,\Lambda^{(t)}_4\}$, and the fact $\lim\limits_{t \rightarrow \infty } \mu^{(t)}_i = \infty$, we further have, by taking limit,
\begin{align*}
    \mU_* &= \mY_* \mU_* + \mE_*; \;\mathcal{Q}_* = \mathcal{W}_*\mU_*\\
    \mY_*\mathbf{1} & = \mathbf 1; \;\mY_* = \mZ_* - \text{diag}(\mZ_*)
    \end{align*}
The other four KKT conditions can be obtained by taking limits over \eqref{eq:Gao1} - \eqref{eq:Gao4}.  The proof of claim 2) is completed.



\noindent Step 3: The proof of claim 3)

We conduct the proof in the following sub-steps.

\noindent(1) $\{\mY^{(t)}\}$ is a Cauchy:

According to the updating rules \eqref{func:objective_1_mu_lambda}, we have
\begin{align*}
&\Lambda_{3}^{(t+1)}=\Lambda_{3}^{(t)}+\mu_{3}^{(t)}(\mY^{(t+1)} \vone-\vone)\\
&\Lambda_{3}^{(t)}=\Lambda_{3}^{(t-1)}+\mu_{3}^{(t-1)}(\mY^{(t)} \vone-\vone)
\end{align*}

This gives
\begin{align*}
\mY^{(t+1)} \vone-\mY^{(t)} \vone = \frac{\Lambda^{(t+1)}_3-\Lambda^{(t)}_3}{\mu^{(t)}_3}  +  \frac{\Lambda^{(t)}_3-\Lambda^{(t-1)}_3}{\mu^{(t-1)}_3}   
\end{align*}

Hence
\begin{align*}
&\|\mY^{(t+1)}-\mY^{(t)}\|^2_F \leq \|(\mY^{(t+1)}-\mY^{(t)})\vone\|^2_F\\
\leq &\frac1{\mu^{2(t)}_3}\|\Lambda^{(t+1)}_3-\Lambda^{(t)}_3\|^2_F + \frac1{\mu^{2(t-1)}_3}\|\Lambda^{(t)}_3-\Lambda^{(t-1)}_3\|^2_F\\
 = &\frac1{\mu^{2(t)}_3}(\|\Lambda^{(t+1)}_3-\Lambda^{(t)}_3\|^2_F + \rho^2 \|\Lambda^{(t)}_3-\Lambda^{(t-1)}_3\|^2_F)\\
=:&\frac1{\mu^{2(t)}_3}C_t
\end{align*}
where $C_t$ is bounded as the sequence $\{\Gamma_t\}$ is bounded.  Then  for any $m > n$, we can establish
\begin{align*}
\|\mY^{(m)} -    \mY^{(n)}\|^2_F &\leq \sum^{m-1}_{t=n}\| \mY^{(t+1)} -    \mY^{(t)}\|^2_F\\
&\leq C \sum^{m-1}_{t=n}\frac1{\mu^{2(t)}_1} = C \sum^{m-1}_{t=n}\frac1{\mu^{(0)}_i \rho^{2t}}
\end{align*}
where $\rho>1$. This completes the proof of Cauchy property of the sequence $\{\mY^{(t)}\}$.

\noindent(2) $\{\mU^{(t)}\}$ is a Cauchy:

The first order condition for the optimal $\mU^{(t+1)}$ reads as, denoting by $\mL^{(t+1)} = \mI - \mY^{(t+1)}$,
\begin{align*}
0 =& \lambda_2\mD(\mU^{(t+1)}-\mX)+\mu^{(t)}_1\mL^{(t+1)\top}(\mL^{(t+1)}\mU^{(t+1)} -\mE^{(t)}) \\
&+\sum_{k,l} \mu^{(t)}_2\mathcal{W}_{k,l}^\top(\mQ^{(t)}_{k,l}-\mathcal{W}_{k,l}\mU^{(t+1)}) \\
&+ \mL^{(t+1)\top}\Lambda^{(t)}_1 +\sum_{k,l}\mathcal{W}_{k,l}\Lambda^{(t)}_{2;k,l}
\end{align*}
Similarly by using the updating rules \eqref{func:objective_1_mu_lambda} and replacing $\mE^{(t)}$ and $\mQ^{(t)}_{k,l}$, we will have
\begin{align*}
0=&\lambda_2\mD(\mU^{(t+1)}-\mX) + \mu^{(t)}_1 \mL^{(t+1)}\mL^{(t+1)\top}(\mU^{(t+1)} - \mU^{(t)})\\
&+\sum_{k,l}\mu^{(t)}_2\mathcal{W}^T_{k,l}\mathcal{W}_{k,l} (\mU^{(t+1)} - \mU^{(t)})\\
-&\rho\mL^{(t+1)}(\Lambda^{(t)}_1 - \Lambda^{(t-1)}_1)-\rho\sum_{k,l}\mathcal{W}^\top_{k,l} (\Lambda^{(t)}_2 - \Lambda^{(t-1)}_2)\\
&+\mu^{(t)}_1 \mL^{(t+1)\top}(\mY^{(t+1)}-\mY^{(t)})\mU^{(t)} + \mL^{(t+1)T}\Lambda^{(t)}_1 \\
&+\sum_{k,l}\mathcal{W}^\top_{k,l}\Lambda^{(t)}_{2;k,l} 
\end{align*}
Assuming $\mu^{(t)}_1 = \mu^{(t)}_2 = \mu^{(t)}$ will give
\begin{align*}
&\|\mU^{(t+1)} -   \mU^{(t)}\|^2_F\\
\leq& \frac1{\mu^{2(t)}}\|(\mL^{(t+1)}\mL^{(t+1)\top} + \sum_{k,l}\mathcal{W}^\top_{k,l}\mathcal{W}_{k,l} )^{-1}\|^2_F  \cdot\\
&\biggl\{\lambda_2 \|\mD(\mU^{(t+1)}-\mX)\|^2_F +\rho^2 \|\mL^{(t+1)}(\Lambda^{(t)}_1 - \Lambda^{(t-1)}_1)\|^2_F\biggr.\\
\\ 
& +  \rho^2\sum_{k,l}\|\mathcal{W}^\top_{k,l} (\Lambda^{(t)}_2 - \Lambda^{(t-1)}_2)\|^2_F +  \|\mL^{(t+1)\top}\Lambda^{(t)}_1\|^2_F\\
&
+ \mu^{(t)}_1 \|\mL^{(t+1)\top}\|^2_F\|(\mY^{(t+1)}-\mY^{(t)})\mU^{(t)}\|^2_F\\
&\biggl.+\sum_{k,l}\|\mathcal{W}^\top_{k,l}\Lambda^{(t)}_{2;k,l}\|^2_F  \biggr\}.
\end{align*}
We consider the second last term where we have seen the term $\mu^{(t)}_1$. First we note that $\|\mL^{(t+1)\top}\|^2_F \leq \|(\mI - \mY^{(t+1)})\vone\|_F$ while $(\mI - \mY^{(t+1)})\vone = (\Lambda^{(t)}_3 - \Lambda^{(t+1)}_3)/\mu^{(t)}_3$. Hence
\[
\mu^{(t)}_1 \|\mL^{(t+1)\top}\|^2_F\leq \frac{\mu^{(t)}_1}{\mu^{(t)}_3}\|\Lambda^{(t)}_3 - \Lambda^{(t+1)}_3\|^2_F = \|\Lambda^{(t)}_3 - \Lambda^{(t+1)}_3\|^2_F
\]
if we take $\mu^{(t)}_1 = \mu^{(t)}_3$. This means there exists a constant $C$ such that
\[
\|\mU^{(t+1)} -   \mU^{(t)}\|^2_F \leq \frac1{\mu^{2(t)}}C.
\]
Similar argument as for Cauchy property for $\{\mY^{(t)}\}$, we can claim $\{\mU^{(t)}\}$ is Cauchy.

\noindent(3) $\{\mE^{(t)}\}, \{\mZ^{(t)}\}, \{\mQ^{(t)}_{k,l}\}$ all are Cauchy:   

As these variables share similar pattern with L1-norm constraint, to save the space, we only take $\{\mE^{(t)}\}$ as an example to show its Cauchy property. We will follow the updating rules  \eqref{func:objective_1_mu_lambda} again.
\begin{align*}
& \|\mE^{(t+1)} - \mE^{(t)}\|^2_F\\
=&\|\mE^{(t+1)}-(\mL^{(t+1)}\mU^{(t+1)} +\Lambda^{(t)}_1/{\mu^{(t)}_1}) \\
&+(\mL^{(t+1)}\mU^{(t+1)} +\Lambda^{(t)}_1/{\mu^{(t)}_1}) -  \mE^{(t)}\|^2_F\\
=&\|\mE^{(t+1)}-(\mL^{(t+1)}\mU^{(t+1)} +\Lambda^{(t)}_1/{\mu^{(t)}_1})\\
&+(\mL^{(t+1)}\mU^{(t+1)} +\Lambda^{(t)}_1/{\mu^{(t)}_1})\\
& + (\Lambda^{(t)}_1 - \Lambda^{(t-1)}_1)/{\mu^{(t-1)}} - \mL^{(t)}\mU^{(t)}\|^2_F
\end{align*}
\begin{align*}
=&\|\mE^{(t+1)}-(\mL^{(t+1)}\mU^{(t+1)} +\frac1{\mu^{(t)}_1}\Lambda^{(t)}_1)- \mU^{(t)} + \mL^{(t)}\mU^{(t)} \\
+&\mU^{(t+1)} - \mY^{(t+1)}\mU^{(t+1)} +\frac1{\mu^{(t)}_1}(\Lambda^{(t)}_1 - \rho(\Lambda^{(t-1)}_1 - \Lambda^{(t)}_1))\|^2_F\\
\leq & \|\mE^{(t+1)}-(\mL^{(t+1)}\mU^{(t+1)} +\frac1{\mu^{(t)}_1}\Lambda^{(t)}_1)\|^2_F \\
&+\|\mY^{(t+1)}\|^2_F\|\mU^{(t+1)}-\mU^{(t)}\|^2_F\\
& +\|\mU^{(t)}\|^2_F\|\mY^{(t+1)}-\mY^{(t)}\|^2_F \\
&+\frac1{\mu^{2(t)}}\|\Lambda^{(t)}_1 - \rho(\Lambda^{(t-1)}_1 - \Lambda^{(t)}_1)\|^2_F.
\end{align*}
Now we prove the first term is second order controlled. $\mE^{(t+1)}$ is obtained by $\ell_1$-threshold operator, so each row element of the first term is less than $1/\mu^{(t)}_1$, hence
\[
\|\mE^{(t+1)}-(\mL^{(t+1)}\mU^{(t+1)} +\frac1{\mu^{(t)}_1}\Lambda^{(t)}_1)\|^2_F\leq \frac{N}{\mu^{2(t)}_1},
\]
where $N$ is the number of rows of $\mE^{(t+1)}$.  We have already proved that $\{\mY^{(t)}\}$ and $\{\mU^{(t)}\}$. With the boundedness of those variable we can claim that $\{\mE^{(t)}\}$ is Cauchy.

Final conclusion: Given the Cauchy property, we can claim that the sequence $\{\mU^{(t)}, \mZ^{(t)}, \mE^{(t)}, \mY^{(t)},\mQ^{(t)}\}$ are convergent to its critical point.

This completes the proof of Theorem 3.

\end{document}